\documentclass{article} %
\usepackage[nonatbib, preprint]{neurips_2021}
\usepackage{times}

\usepackage[numbers]{natbib}

\usepackage{amsmath,amsfonts,bm}

\def\eqref#1{equation~\ref{#1}}

\def\1{\bm{1}}

\DeclareMathAlphabet{\mathsfit}{\encodingdefault}{\sfdefault}{m}{sl}
\SetMathAlphabet{\mathsfit}{bold}{\encodingdefault}{\sfdefault}{bx}{n}

\newcommand{\E}{\mathbb{E}}

\DeclareMathOperator*{\argmax}{arg\,max}
\DeclareMathOperator*{\argmin}{arg\,min}

\usepackage{mdpn}
\usepackage{todonotes}

\usepackage[utf8]{inputenc} %
\usepackage[T1]{fontenc}    %
\usepackage{hyperref}       %
\usepackage{url}            %
\usepackage{booktabs}       %
\usepackage{amsfonts}       %
\usepackage{nicefrac}       %
\usepackage{microtype}      %
\usepackage{bm,bbm,amsmath,amsfonts,amssymb}
\usepackage{xspace}
\usepackage{xcolor}
\usepackage{eucal}
\usepackage{amsmath}
\usepackage{mathtools}  %
\usepackage{appendix}

\usepackage{enumitem}

\usepackage{algorithm}
\usepackage[ruled,vlined,algo2e,linesnumbered]{algorithm2e}

\usepackage{hyperref}

\usepackage[center]{caption}
\usepackage{subcaption}
\usepackage{wrapfig}

\usepackage{amsthm}	%
\newtheorem{definition}{Definition}[]

\newtheorem{lemma}{Lemma}[]

\usepackage{theoremref} %
\usepackage{thmtools}
\usepackage{thm-restate} %
\usepackage{amsxtra} %
\usepackage{hyperref}
\usepackage{url}

\usepackage{wrapfig}

\usepackage{textcase}

\usepackage{breqn}

\DeclareMathOperator*{\defd}{:=}
\DeclareMathOperator*{\gset}{\mathcal{G}}

\newcommand{\aim}{\textsc{aim}}

\title{Adversarial Intrinsic Motivation for Reinforcement Learning}

\author{Ishan Durugkar  \\
Department of Computer Science\\
The University of Texas at Austin\\
Austin, TX, USA 78703 \\
\texttt{ishand@cs.utexas.edu} \\
\And
Mauricio Tec  \\
Department of Statistics and Data Sciences\\
The University of Texas at Austin\\
Austin, TX, USA 78703 \\
\texttt{mauriciogtec@utexas.edu} \\
\And
Scott Niekum  \\
Department of Computer Science\\
The University of Texas at Austin\\
Austin, TX, USA 78703 \\
\texttt{sniekum@cs.utexas.edu} \\
\And
Peter Stone \\
Department of Computer Science\\
The University of Texas at Austin\\
Austin, TX, USA 78703  and \\
Sony AI \\
\texttt{pstone@cs.utexas.edu}
}

\begin{document}

\maketitle

\begin{abstract}
Learning with an objective to minimize the mismatch with a reference distribution has been shown to be useful for generative modeling and imitation learning.
In this paper, we investigate whether one such objective, the Wasserstein-1 distance between a policy's state visitation distribution and a target distribution, can be utilized effectively for reinforcement learning (RL) tasks.
Specifically, this paper focuses on goal-conditioned reinforcement learning where the idealized (unachievable) target distribution has full measure at the goal.
This paper introduces a quasimetric specific to Markov Decision Processes (MDPs) and uses this quasimetric to estimate the above Wasserstein-1 distance.
It further shows that the policy that minimizes this Wasserstein-1 distance is the policy that reaches the goal in as few steps as possible.
Our approach, termed Adversarial Intrinsic Motivation (\textsc{aim}), estimates this Wasserstein-1 distance through its dual objective and uses it to compute a supplemental reward function.
Our experiments show that this reward function changes smoothly with respect to transitions in the MDP and directs the agent's exploration to find the goal efficiently.
Additionally, we combine \aim\ with Hindsight Experience Replay (\textsc{her}) and show that the resulting algorithm accelerates learning significantly on several simulated robotics tasks when compared to other rewards that encourage exploration or accelerate learning.
\end{abstract}

\section{Introduction}

Reinforcement Learning (RL) \citep{sutton2018reinforcement} deals with the problem of learning a policy to accomplish a given task in an optimal manner.
This task is typically communicated to the agent by means of a reward function.
If the reward function is sparse \citep{Arumugam2021AnIP} (e.g., most transitions yield a reward of $0$), much random exploration might be needed before the agent experiences any signal relevant to learning \citep{bellemare2016unifying,arjona2019rudder}.

Some of the different ways to speed up reinforcement learning by modifying or augmenting the reward function are shaped rewards \citep{ng1999policy}, redistributed rewards \citep{arjona2019rudder}, intrinsic motivations \citep{baldassarre_intrinsic_2013,singh_intrinsically_2010, sorg2010reward, sorg2010internal, niekum2010evolved, oudeyer2009intrinsic}, and learned rewards \citep{zheng_learning_2018, niekum2010evolved}.
Unfortunately, the optimal policy under such modified rewards might sometimes be different than the optimal policy under the task reward \cite{ng1999policy,clark_faulty_2016}.
The problem of learning a reward signal that speeds up learning by communicating \emph{what to do} but does not interfere by specifying \emph{how to do it} is thus a useful and complex one \citep{zheng_what_2020}.

This work considers whether a task-dependent reward function learned based on the distribution mismatch between the agent's state visitation distribution and a target task (expressed as a distribution) can guide the agent towards accomplishing this task.
Adversarial methods to minimize distribution mismatch have been used with great success in generative modeling \citep{goodfellow_generative_2014} and imitation learning \citep{ho2016generative, fu_learning_2018,xiao_wasserstein_2019, torabi_generative_2019, ghasemipour2020divergence}.
In both these scenarios, the task is generally to minimize the mismatch with a target distribution induced by the data or expert demonstrations.
Instead, we consider the task of goal-conditioned RL, where the ideal target distribution assigns full measure to a goal state.
While the agent can never match this idealized target distribution perfectly unless starting at the goal, intuitively, minimizing the mismatch with this distribution should lead to trajectories that maximize the proportion of the time spent at the goal, thereby prioritizing transitions essential to doing so.

The theory of optimal transport \citep{villani2008optimal} gives us a way to measure the distance between two distributions (called the Wasserstein distance) even if they have disjoint support.
Previous work \citep{arjovsky_wasserstein_2017, gulrajani_improved_2017} has shown how a neural network approximating a potential function may be used to estimate the Wasserstein-1 distance using its dual formulation, but assumes that the metric space this distance is calculated on is Euclidean.
A Euclidean metric might not be the appropriate metric to use in more general RL tasks however, such as navigating in a maze or environments where the state features change sharply with transitions in the environment.

This paper introduces a quasimetric tailored to Markov Decision Processes (MDPs), the time-step metric, to measure the Wasserstein distance between the agent's state visitation distribution and the idealized target distribution.
While this time-step metric could be an informative reward on its own, estimating it is a problem as hard as policy evaluation \citep{guillot2020stochastic}.
Instead, we show that the dual objective, which maximizes difference in potentials while utilizing the structure of this quasimetric for the necessary regularization, can be optimized through sampled transitions.

We use this dual objective to estimate the Wasserstein-1 distance and propose a reward function based on this estimated distance.
An agent that maximizes returns under this reward minimizes this Wasserstein-1 distance.
The competing objectives of maximizing the difference in potentials for estimating the Wasserstein distance and minimizing it through reinforcement learning on the subsequent reward function leads to our algorithm, Adversarial Intrinsic Motivation (\aim).

Our analysis shows that if the above Wasserstein-1 distance
is computed using the time-step metric, then minimizing it leads to a policy that reaches the goal in the minimum expected number of steps.
It also shows that if the environment dynamics are deterministic, then this policy is the optimal policy.

In practice, minimizing the Wasserstein distance works well even when the environment dynamics are stochastic.
Our experiments show that \aim\ learns a reward function that changes smoothly with transitions in the environment.
We further conduct experiments on the family of goal-conditioned reinforcement learning problems \cite{andrychowicz_hindsight_2018, Schaul2015UniversalVF} and show that \aim\ when used along with hindsight experience replay (\textsc{her}) greatly accelerates learning of an effective goal-conditioned policy compared to learning with \textsc{her} and the sparse task reward.
Further, our experiments show that this acceleration is similar to the acceleration observed by using the actual distance to the goal as a dense reward.

\section{Related Work}

We highlight the related work based on the various aspects of learning that this work touches,
namely intrinsic motivation, goal-conditioned reinforcement learning, and adversarial imitation learning.

\subsection{Intrinsic Motivation}

Intrinsic motivations \citep{baldassarre_intrinsic_2013, oudeyer2009intrinsic, oudeyer2008can} are rewards presented by an agent to itself in addition to the external task-specific reward.
Researchers have pointed out that such intrinsic motivations are a characteristic of naturally intelligent and curious agents \citep{gottlieb2013information, baldassarre2011intrinsic, baldassarre2014intrinsic}.
Intrinsic motivation has been proposed as a way to encourage RL agents to learn skills \citep{barto2004intrinsically, barto2005intrinsic, singh2005intrinsically, santucci2013best} that might be useful across a variety of tasks, or as a way to encourage exploration \citep{bellemare2016unifying, csimcsek2006intrinsic, baranes2009r, forestier2017intrinsically}.
The optimal reward framework \citep{singh_intrinsically_2010, sorg2010internal} and shaped rewards \citep{ng1999policy} (if generated by the agent itself) also consider intrinsic motivation as a way to assist an RL agent in learning the optimal policy for a given task.
Such an intrinsically motivated reward signal has previously been learned through various methods such as evolutionary techniques \citep{niekum2010evolved, schembri2007evolving}, meta-gradient approaches \citep{sorg2010reward, zheng_learning_2018,zheng_what_2020}, and others.
The Wasserstein distance has been used to present a valid reward for imitation learning \citep{xiao_wasserstein_2019, dadashi_primal_2020} as well as program synthesis \citep{ganin2018synthesizing}.

\subsection{Goal-Conditioned Reinforcement Learning} \label{sec:rel_goal}

Goal-conditioned reinforcement learning \citep{kaelbling1993learning} can be considered a form of multi-task reinforcement learning \citep{caruana1997multitask} where the agent is given the goal state it needs to reach at the beginning of every episode, and the reward function is sparse with a non-zero reward only on reaching the goal state.
\textsc{uvfa} \citep{Schaul2015UniversalVF}, \textsc{her} \citep{andrychowicz_hindsight_2018}, and others \citep{zhang_automatic_2020,Ding2019GoalconditionedIL} consider this problem of reaching certain states in the environment.
Relevant to our work, \citet{Venkattaramanujam2019SelfsupervisedLO} learns a distance between states using a random walk that is then used to shape rewards and speed up learning, but requires goals to be visited before the distance estimate is useful.
DisCo RL \citep{Nasiriany:EECS-2020-151} extends the idea of goal-conditioned RL to distribution-conditioned RL.

Contemporaneously, \citet{ eysenbach2021replacing,eysenbach2021clearning} has proposed a method which considers goals and examples of success and tries to predict and maximize the likelihood of seeing those examples under the current policy and trajectory.
For successful training, this approach needs the agent to actually experience the goals or successes.
Their solution minimizes the Hellinger distance to the goal, a form of $f$-divergence.
\aim\ instead uses the Wasserstein distance which is theoretically more informative when considering distributions that are disjoint, and does not require the assumption that the agent has already reached the goal through random exploration.
Our experiments in fact verify the hypothesis that \aim\ induces a form of directed exploration in order to reach the goal.

\subsection{Adversarial Imitation Learning and Minimizing Distribution Mismatch} \label{sec:AIL}

Adversarial imitation learning \citep{ho2016generative,fu_learning_2018,xiao_wasserstein_2019,torabi_generative_2019,ghasemipour2020divergence} has been shown to be an effective method to learn agent policies that minimize distribution mismatch between an agent's state-action visitation distribution and the state-action visitation distribution induced by an expert's trajectories.
In most cases this distribution that the expert induces is achievable by the agent and hence these techniques aim to match the expert distribution exactly.
In the context of goal-conditioned reinforcement learning, GoalGAIL \citep{Ding2019GoalconditionedIL} uses adversarial imitation learning with a few expert demonstrations to accelerate the learning of a goal-conditioned policy.
In this work, we focus on unrealizable target distributions that cannot be completely matched by the agent, and indeed, are not induced by any trajectory distribution.

\textsc{fairl} \citep{ghasemipour2020divergence} is an adversarial imitation learning technique which minimizes the Forward KL divergence and has been shown experimentally to cover some hand-specified state distributions, given a smoothness regularization as used by WGAN \citep{gulrajani_improved_2017}.
$f$-\textsc{irl} \citep{ni2020fIRL} learns a reward function where the optimal policy matches the expert distribution under the more general family of $f$-divergences.
Further, techniques beyond imitation learning \citep{lee2019efficient, hazan2019provably} have looked at matching a uniform distribution over states to guarantee efficient exploration.

\section{Background}

In this section we first set up the goal-conditioned reinforcement learning problem, and then give a brief overview of optimal transport.

\subsection{Goal-Conditioned Reinforcement Learning} \label{sec:goal}

Consider a goal-conditioned MDP as the tuple $\langle \sset, \aset, \gset, P, \rho_0, \sigma, \gamma \rangle$ with discrete state space $\sset$, discrete action space $\aset$, a subset of states which is the goal set  $\gset \subseteq \sset$, and transition dynamics $P: \sset \times \aset \times \gset \longmapsto \Delta(\sset)$ ($\Delta(\cdot)$ is a distribution over a set) which might vary based on the goal (see below).
$\rho_0: \Delta(\sset)$ is the starting state distribution, and $\sigma: \Delta(\gset)$ is the distribution a goal is drawn from.
$\gamma \in [0, 1)$ is the discount factor.
We use discrete states and actions for ease of exposition, but our idea extends to continuous states and actions, as seen in the experiments.

At the beginning of an episode, the starting state is drawn from $\rho_0$ and the goal for that episode is drawn from $\sigma$.
The reward function $r: \sset \times \aset \times \sset \times \gset \longmapsto \mathbb{R}$ is deterministic, and $r(s_t, a_t, s_{t+1}|s_g) \defd \mathbb{I}[s_{t+1}=s_g]$.
That is, there is a positive reward when an agent reaches the goal ($s_{t+1} = s_g$), and $0$ everywhere else.
Since the goal is given to the agent at the beginning of the episode, in goal-conditioned RL the agent knows what this task reward function is (unlike the more general RL problem).
The transition dynamics are goal-conditioned as well, with an automatic transition to an absorbing state $\Bar{s}$ on reaching the goal $s_g$ and then staying in that state with no rewards thereafter ($P(\Bar{s}|s_g, a, s_g) = 1 \;\forall\; a \in \aset$ and $P(\Bar{s}|\Bar{s}, a, s_g) = 1 \;\forall\; a \in \aset$).
In short, the episode terminates on reaching the goal state.

The agent takes actions in this environment based on a policy $\pi \in \Pi: \sset \times \gset \longmapsto \Delta(\aset)$.
The return $H_g$ for an episode with goal $s_g$ is the discounted cumulative reward over that episode $H_g = \sum_{t=0}^{\infty} \gamma^t r(s_t, a_t, s_{t+1} | s_g)$, where $s_0 \sim \rho_0$, $a_t \sim \pi(\cdot|s_t, s_g)$, and  $s_{t+1} \sim P(\cdot|s_t, a_t, s_g)$.
The agent aims to find the policy $\pi^* = \argmax_{\pi \in \Pi} \mathbb{E}_{g \in \mathcal{G}} \mathbb{E}_{s_0 \sim \rho_0} \mathbb{E}_{\pi}[H_g]$ that maximizes the expected returns in this goal-conditioned MDP.
For a policy $\pi$, the agent's goal-conditioned state distribution $\rho_{\pi}(s|s_g) = \mathbb{E}_{s_0 \sim \rho_0}[(1 - \gamma)\sum_{t=0}^{\infty} \gamma^t P(s_t=s|\pi, s_g)]$.
Overloading the terminology a bit, we also define the goal-conditioned target distribution $\rho_g(s| s_g) = \delta(s_g)$, a Dirac measure at the goal state $s_g$.

While learning using traditional RL paradigms is possible in goal-conditioned RL, there has also been previous work (Section \ref{sec:rel_goal}) on leveraging the structure of the problem across goals.
Hindsight Experience Replay (\textsc{her}) \citep{andrychowicz_hindsight_2018} attempts to speed up learning in this sparse reward setting by taking episodes of agent interactions, where they might not have reached the goal specified for that episode, and relabeling the transitions with the goals that \emph{were} achieved during the episode.
Off-policy learning algorithms are then used to learn from this relabeled experience.

\subsection{Optimal Transport and Wasserstein-1 Distance} \label{sec:opt}

The theory of optimal transport \citep{villani2008optimal,bousquet_optimal_2017} considers the question of how much work must be done to transport one distribution to another optimally, where this notion of work is defined by the use of a ground metric $d$.
More concretely, consider a metric space ($\mathcal{M}$, $d$) where $\mathcal{M}$ is a set and $d$ is a metric on $\mathcal{M}$ (Definitions in Appendix \ref{sec:metrics}).
For two distributions $\mu$ and $\nu$ with finite moments on the set $\mathcal{M}$, the Wasserstein-$p$ distance is denoted by: 
\begin{align} \label{eqn:w-p}
    W_p(\mu, \nu) \defd \inf_{\zeta \in Z(\mu, \nu)} \E_{(X,Y)\sim \zeta}[d(X, Y)^p]^{1/p}
\end{align}

where $Z$ is the space of all possible couplings, i.e.\ joint distributions $\zeta \in \Delta(\mathcal{M}\times\mathcal{M})$ whose marginals are $\mu$ and $\nu$ respectively.
Finding this optimal coupling tells us what is the least amount of work, as measured by $d$, that needs to be done to convert $\mu$ to $\nu$.
This Wasserstein-$p$ distance can then be used as a cost function (negative reward) by an RL agent to match a given target distribution \citep{xiao_wasserstein_2019,dadashi_primal_2020,ganin2018synthesizing}.

\begin{wrapfigure}{l}{0.32\textwidth} 
  \vspace{-20pt}
  \begin{center}
      \includegraphics[width=0.3\textwidth]{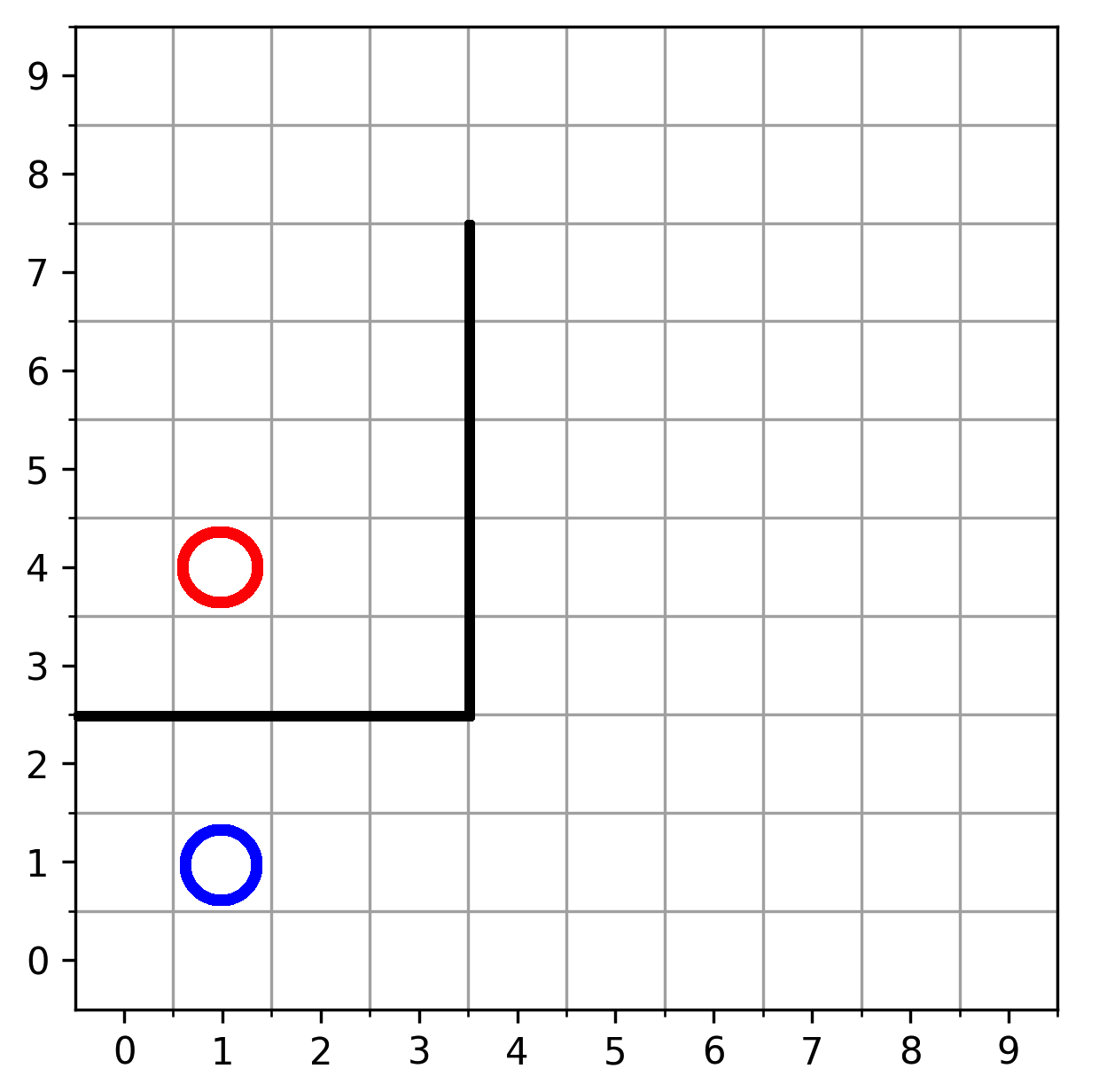}
  \end{center}
  \caption{Grid world example} \label{fig:example}
  \vspace{-5pt}
\end{wrapfigure}

Finding the ideal coupling above is generally considered intractable.
However, if what we need is an accurate estimate of the Wasserstein distance and not the optimal transport plan we can turn our attention to the dual form of the Wasserstein-1 distance.
The Kantorovich-Rubinstein duality \citep{villani2008optimal, peyre_computational_2020} for the Wasserstein-1 distance (which we refer to simply as the Wasserstein distance hereafter) on a ground metric $d$ gives us:
\begin{align} \label{eqn:KRdual}
    W_1(\mu, \nu) = \sup_{\text{Lip}(f) \leq 1} \mathbb{E}_{y \sim \nu}\left[f(y)\right] - \mathbb{E}_{x \sim \mu} \left[f(x)\right]
\end{align}

where the supremum is over all $1$-Lipschitz functions $f: \mathcal{M} \longmapsto \mathbb{R}$ in the metric space.
Importantly, \citet{jevtic2018combinatorial} has recently shown that this dual formulation extends to quasimetric spaces as well.
More details, such as the definition of the Lipschitz constant of a function and special cases used in WGAN \citep{arjovsky_wasserstein_2017, gulrajani_improved_2017, ghasemipour2020divergence} are elaborated in Appendix \ref{app:opt}.
One last note of importance is that the Lipschitz constant of the potential function $f$ is computed based on the ground metric $d$.

\section{Time-Step Metric} \label{sec:cmetric}

The choice of the ground metric $d$ is important when computing the Wasserstein distance between two distributions.
That is, if we want the Wasserstein distance to give an estimate of the work needed to transport the agent's state visitation distribution to the goal state,
the ground metric should incorporate a notion of this work.

Consider the grid world shown in \autoref{fig:example}, where a wall (bold line) marks an impassable barrier in part of the state space.
If the states are specified by their Cartesian coordinates on the grid, the Manhattan distance between the states specified by the blue and red circles is not representative of the optimal cost to go from one to the other.
This mismatch would lead to an underestimation of the work involved if the two distributions compared were concentrated at those two circles.
Similarly, there will be errors in estimating the Wasserstein distance if the grid world is toroidal (where an agent is transported to the opposite side of the grid if it walks off one side) or if the transitions are asymmetric (windy grid world \citep{sutton2018reinforcement}).

To estimate the work needed to transport measure in an MDP when executing a policy $\pi$, we consider a \emph{quasimetric} -- a metric that does not need to be symmetric -- dependent on the number of transitions experienced before reaching the goal when executing that policy.

\begin{definition} \label{def:tmetric}
The \textbf{time-step metric} $d^\pi_T$ in an MDP with state space $\sset$, action space $\aset$, transition function $P$, and agent policy $\pi$ is a quasimetric
where the distance from state $s \in \sset$ to state $s_g \in \sset$ is based on the expected number of transitions under policy $\pi$.
\begin{align*}
    d^\pi_T(s, s_g) \defd \mathbb{E}\; \left[T(s_g |\pi, s)\right]
\end{align*}
where $T(s_g| \pi, s)$ is the random variable for the first time-step that state $s_g$ is encountered by the agent after starting in state $s$ and following policy $\pi$.
\end{definition}

This quasimetric has the property that the per step cost is uniformly $1$ for all transitions except ones from the goal to the absorbing state (and the absorbing state to itself), which are $0$.
Thus, it can be written recursively as:
\begin{align} \label{eqn:rec_time_d}
    d^\pi_T(s, s_g) = \begin{cases} 0 & \text{if } s = s_g \\
    1 + \mathop{\mathbb{E}}_{a \sim \pi(\cdot|s, s_g)}\mathop{\mathbb{E}}_{s' \sim  P(\cdot|s, a, s_g)}\left[  d^\pi_T(s', s_g) \right]& \text{otherwise} \end{cases}
\end{align}

Recall that in order to estimate the Wasserstein distance using the dual (\autoref{eqn:KRdual}) in a metric space where the ground metric $d$ is this time-step metric, the potential function $f: \sset \longmapsto \mathbb{R}$ needs to be $1$-Lipschitz with respect to $d^\pi_T$.
In \autoref{app:quasi_lip} we prove that $L$-Lipschitz continuity can be ensured by enforcing that the difference in values of $f$ on expected transitions from every state are bounded by $L$, implying
\begin{align} \label{eqn:lip_mdp}
 \text{Lip}(f) \leq \sup_{s \in \sset}\left\{\mathop{\mathbb{E}}_{a \sim \pi(\cdot|s, s_g)} \mathop{\mathbb{E}}_{s' \sim P(\cdot|s, a, s_g)}\left[\left\lvert f(s) - f(s')\right\rvert\right] \right\} .
\end{align}
Note that finding a proper way to enforce the Lipschitz constraint in adversarial methods remains an open problem \citep{liu2020lipschitz}.
However, for the time-step metric considered here, \eqref{eqn:lip_mdp} is one elegant way of doing so.
By ensuring that the Kantorovich potentials do not drift too far from each other on expected transitions under agent policy $\pi$ in the MDP, the conditions necessary for the potential function to estimate the Wasserstein distance can be maintained \citep{villani2008optimal,arjovsky_wasserstein_2017}.
Finally, the minimum distance $d^\blacklozenge_T$ from state $s$ to a given goal state $s_g$ (corresponding to policy $\pi^{\blacklozenge}$) is defined by the Bellman optimality condition (\autoref{eqn:opt_dt} in \autoref{app:proofs}).

Consider how the time-step distance to the goal and the value function for goal-conditioned RL relate to each other.
When the reward is $0$ everywhere except for transitions to the goal state, the value becomes $V^\pi(s|s_g) = \mathbb{E} \left[ \gamma^{T(s_g|\pi, s)} \right]$.
$d^\pi_T(s_0, s_g)$ and $V(s_0|s_g)$ are related as follows.

\begin{restatable}{proposition}{lowerbound}
\label{prop:lower_b}
A lower bound on the value of any state under a policy $\pi$ can be expressed in terms of the time-step distance from that state to the goal: $V(s_0|s_g) \geq \gamma^{d^\pi_T(s_0, s_g)}$.
\end{restatable}

The proofs for all theoretical results are in \autoref{app:proofs}.
The Jensen gap $\Delta^\pi_{\text{Jensen}}(s):=V^\pi(s|s_g) -\gamma^{d^\pi_T(s, s_g)} $ describes the sharpness of the lower bound in the proposition above and it is zero if and only if $\mathrm{Var}(T(s_g|\pi, s))=0$ \citep{liao2018sharpening}.
From this line of reasoning, we deduce the following proposition:
\begin{restatable}{proposition}{policyCorrespondence}
\label{prop:corr}
If the transition dynamics are deterministic, the policy that maximizes expected return is the policy that minimizes the time-step metric ($\pi^* = \pi^\blacklozenge$).
\end{restatable}

\section{Wasserstein-1 Distance for Goal-Conditioned Reinforcement Learning}
\label{sec:wass_theory}

In this section we consider the problem of goal-conditioned reinforcement learning. In Section \ref{sec:wass_opt} we analyze the Wasserstein distance computed under the time-step metric $d^\pi_T$.
Section \ref{sec:alg} proposes an algorithm, Adversarial Intrinsic Motivation (\textsc{aim}), to learn the potential function for the Kantorovich-Rubinstein duality used to estimate the Wasserstein distance, and giving an intrinsic reward function used to update the agent policy in tandem.

\subsection{Wasserstein-1 Distance under the Time-Step Metric} \label{sec:wass_opt}

From Sections \ref{sec:opt} and \ref{sec:cmetric} the Wasserstein distance under the time-step metric $d^\pi_T$ of an agent policy $\pi$ with visitation measure $\rho_\pi$ to a particular goal $s_g$ and its distribution $\rho_g$ can be expressed  as:
\begin{align} \label{eqn:dt_W1}
    W_1^\pi(\rho_\pi, \rho_g) = \textstyle{\sum_{s \in \sset}} \rho_\pi(s|s_g) d^\pi_T(s, s_g)
\end{align}
where $W_1^\pi$ refers to the Wasserstein distance with the ground metric $d^{\pi}_T$.

The following proposition shows that the Wasserstein distance decreases as $d^{\pi}_T(s, s_g)$ decreases, while also revealing a surprising connection with the Jensen gap.

\begin{restatable}{proposition}{analytical}
\label{prop:analytical}
For a given policy $\pi$, the Wasserstein distance of the state visitation measure of that policy from the goal state distribution $\rho_g$ under the ground metric $d^{\pi}_T$ can be written as
\begin{align}
    W_{1}^{\pi}(\rho_\pi, \rho_g) = \mathop{\E}_{s_0 \sim \rho_0} \left[ h(d^{\pi}_T(s_0, s_g)) + \frac{\gamma}{1 - \gamma}(\Delta^\pi_{\text{Jensen}}(s_0) - 1)\right]
\end{align}
where $h$ is an increasing function of $d_T^\pi$.
\textbf{}\end{restatable}

The first component in the above analytical expression shows that the Wasserstein distance depends on the expected number of steps, decreasing if the expected distance decreases.
The second component shows the risk-averse nature of the Wasserstein distance.
Concretely, the bounds for the Jensen inequality given by \citet{liao2018sharpening} imply that there are non-negative constants $C_1=C_1(d_T^\pi, \gamma)$ and $C_2=C_2(d_T^\pi, \gamma)$ depending only on the expected distance and discount factor such that
$$
C_1\mathrm{Var}(T(s_g|\pi, s)) \leq \Delta^\pi_{\text{Jensen}}(s) \leq C_2 \mathrm{Var}(T(s_g|\pi, s)).
$$
From the above, we can deduce that a policy with lower variance will have lower Wasserstein distance when compared to a policy with same expected distance from the start but higher variance.
The relation between the optimal policy in goal-conditioned RL and the Wasserstein distance can be made concrete if we consider deterministic dynamics.
\begin{restatable}{theorem}{aimsearch}
\label{thm:aim_search}
If the transition dynamics are deterministic, the policy that minimizes the Wasserstein distance over the time-step metrics in a goal-conditioned MDP (see \eqref{eqn:dt_W1}) is the optimal policy.
\end{restatable}

\subsection{Adversarial Intrinsic Motivation to minimize Wasserstein-1 Distance} \label{sec:alg}

The above section makes it clear that minimizing the Wasserstein distance to the goal will lead to a policy that reaches the goal in as few steps as possible in expectation.
If the dynamics of the MDP are deterministic, this policy will also be optimal.
Note that the dual form (\autoref{eqn:KRdual}) can be used to estimate the distance, \emph{even if the ground metric $d^\pi_T$ is not known}.
The smoothness requirement on the potential function $f$ can be ensured with the constraint in \autoref{eqn:lip_mdp} on all states and subsequent transitions expected under the agent policy.

Now consider the full problem.
The reinforcement learning algorithm aims to learn a goal-conditioned policy with parameters $\theta \in \Theta$ whose state visitation distribution $\rho_\theta$ minimizes the Wasserstein distance to a goal-conditioned target distribution $\rho_g$ for a given goal $s_g \sim \sigma$.
\aim\ leverages the presence of the set of goals that the agent should be capable of reaching with a goal-conditioned potential function $f_{\phi}: \sset \times \gset \longmapsto \mathbb{R}$ with parameters $\phi \in \Phi$.
These objectives of the potential function and the agent can be expressed together using the following adversarial objective:
\begin{align} \label{eqn:pol_search}
    \min_{\theta \in \Theta} \max_{\phi \in \Phi} \mathop{\mathbb{E}}_{s_g \sim \sigma}\left[f_{\phi}(s_g,s_g) - \mathop{\mathbb{E}}_{s \sim \rho_{\theta}} [f_{\phi}(s, s_g)]\right]
\end{align}
where the potential function $f_\phi$ is $1$-Lipschitz over the state space.
Combining the objectives in Equations \ref{eqn:pol_search} and \ref{eqn:lip_mdp}, the loss for the potential function $f_\phi$ then becomes:
\begin{align} 
    L_f \defd &\mathop{\mathbb{E}}_{s_g \sim \sigma} \left[ - f_{\phi}(s_g, s_g) + \mathop{\mathbb{E}}_{s \sim \rho_{\theta}} [f_\phi(s, s_g) ]\right] +  \nonumber \\
    \lambda & \mathop{\mathbb{E}}_{(s, a, s', s_g) \sim \mathcal{D}} \left[(\max(|f_\phi(s, s_g) - f_\phi(s', s_g)| - 1, 0))^2 \right] \label{eqn:f_loss}
\end{align}

Where the distribution $\mathcal{D}$ should ideally contain all states in $\sset$, expected goals in $\gset$, and the transitions according to the agent policy $\pi_\theta$ and transition function $P$.
Such a distribution is difficult to obtain directly.
\aim\ approximates it with a small replay buffer of transitions from recent episodes experienced by the agent, and relabels these episodes with achieved goals (similar to \textsc{her} \citep{andrychowicz_hindsight_2018}).
Such an approximation does not respect the discounted measure of states later on in an episode, but is consistent with how other approaches in deep reinforcement learning tend to approximate the state visitation distribution, especially for policy gradient approaches \citep{nota2020PolicyGradient}.
While it does not include all states and all goals, we see empirically that the above approximation works well.

Now we turn to the reward function that should be presented to the agent such that maximizing the return will minimize the Wasserstein distance.
The Wasserstein discriminator is a potential function \citep{ng1999policy} (its value depends on the state).
It can thus be used to create a shaped reward $\hat{r}(s, a, s', s_g) = r(s, a, s'|s_g) + \gamma f_\phi(s', s_g) - f_\phi(s, s_g)$ without risk of changing the optimal policy.
Alternatively, we can explicitly minimize samples of the Wasserstein distance: $\hat{r}(s, a, s', s_g) = f_\phi(s', s_g) - f_\phi(s_g, s_g)$.
Finally, instead of the second term $f_\phi(s_g, s_g)$, we can just use a constant bias term.
In practice, all these choices work well, and the experiments use the latter (with $b = \max_{s \in \sset} f_\phi(s,s_g)$) to reduce variance in $\hat{r}$.
\begin{align} \label{eqn:reward}
    \hat{r}(s, a, s', s_g) = f_\phi(s', s_g) - b
\end{align}

The basic procedure to learn and use adversarial intrinsic motivation (\textsc{aim}) is laid out in Algorithm \ref{alg}, and also includes how to use this algorithm in conjunction with \textsc{her}.
If not using \textsc{her}, Line \ref{alg:her} where hindsight goals are added to the replay buffer can be skipped.

\section{Experiments} \label{sec:exp}

Our experiments evaluate the extent to which the reward learned through \textsc{aim} is useful as a proxy for the environment reward signal, or in tandem with the environment reward signal.
In particular, we ask the following questions:
\begin{itemize}[noitemsep,topsep=0pt,leftmargin=5.5mm]
    \item Does \aim\ speed up learning of a policy to get to a single goal compared to learning with a sparse reward?
    \item Does the learned reward function qualitatively guide the agent to the goal?
    \item Does \aim\ work well with stochastic transition dynamics or sharp changes in the state features?
    \item Does \aim\ generalize to a large set of goals and continuous state and action spaces?
\end{itemize}
Our experiments suggest that the answer to all 4 questions is ``yes'',
with the first three questions tested in the grid world presented in Figure \ref{fig:example} where the goal is within a room, and the agent has to go around the room from its start state to reach the goal.
Goal-conditioned tasks in the established Fetch robot domain show that \textsc{aim} also accelerates learning across multiple goals in continuous state and action spaces.

This section compares an agent learning with a reward learned through \aim\  with other intrinsic motivation signals that induce general exploration or shaped rewards that try to guide the agent to the goal.
The experiments show that \aim\ guides the agent's exploration more efficiently and effectively than a general exploration bonus, and adapts to the dynamics of the environment better than other techniques we compare to.
As an overview, the baselines we compare to are:
\begin{itemize}[noitemsep,topsep=0pt,leftmargin=5.5mm]
    \item \textbf{RND}: with random network distillation (\textsc{rnd}) \citep{burda2018rnd} used to provide a general exploration bonus.
    \item \textbf{MC}: with the distance between states learned through regression of Monte Carlo rollouts of the agent policy, similar to \citet{hartikainen2019ddl}.
    \item \textbf{SMiRL}: \textsc{sm}i\textsc{rl} \citep{berseth2019smirl} is used to provide a bonus intrinsic motivation reward that minimizes the overall surprise in an episode.
    \item \textbf{DiscoRL} The DiscoRL \citep{Nasiriany:EECS-2020-151} approach presents a reward to maximize the likelihood of a target distribution (normal distribution at the goal). In practice this approach is equivalent to a negative L2 distance to the goal, which we compare to in the grid world domain.
    \item \textbf{GAIL}: additional \textsc{gail} \citep{ho2016generative} rewards using trajectories relabeled with achieved goals considered as having come from the expert in hindsight. This baseline is compared to in the Fetch robot domain, since that is the domain where we utilize hindsight relabeling.
\end{itemize}

\paragraph{Grid World}
In this task, the goal is inside a room and the agent's starting position is such that it needs to navigate around the room to find the doorway and be able to reach the goal.
The agent can move in the $4$ cardinal directions unless blocked by a wall or the edge of the grid.
The agent policy is learned using soft Q-learning \citep{haarnoja2017reinforcement} with no hindsight goals used for this experiment.

The agent's state visitation distribution after just $100$ Q-function updates when using \aim-learned rewards is shown in Figure \ref{fig:grid_policy} and the learned rewards for each state are plotted in Figure \ref{fig:reward}.
The state visitation when learning with the true task reward shows that the agent is unable to learn a policy to the goal (Figure \ref{fig:task_policy}).
These figures show that \aim\ enables the agent to reach the goal and learn the required policy quickly, while learning with the sparse task reward fails to do so. %

\begin{figure}[b]
    \centering
    \begin{subfigure}[t]{.3\textwidth}
        \centering
        \includegraphics[width=\linewidth]{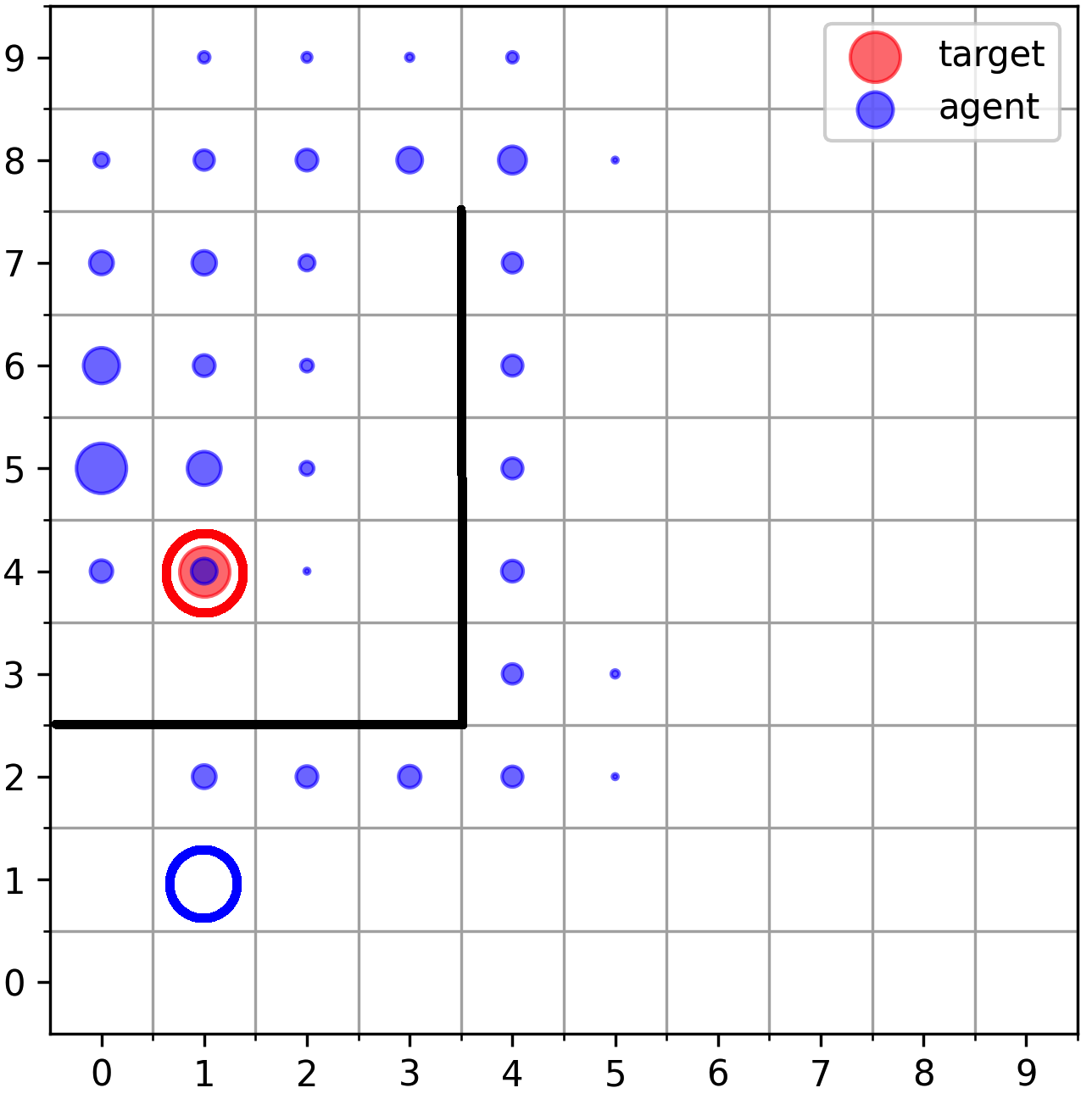}
        \caption{Agent state visitation learning with \textsc{aim} rewards (100 iterations)}
        \label{fig:grid_policy}
    \end{subfigure}%
    \hfill
    \begin{subfigure}[t]{.35\textwidth}
        \centering
        \includegraphics[width=\linewidth]{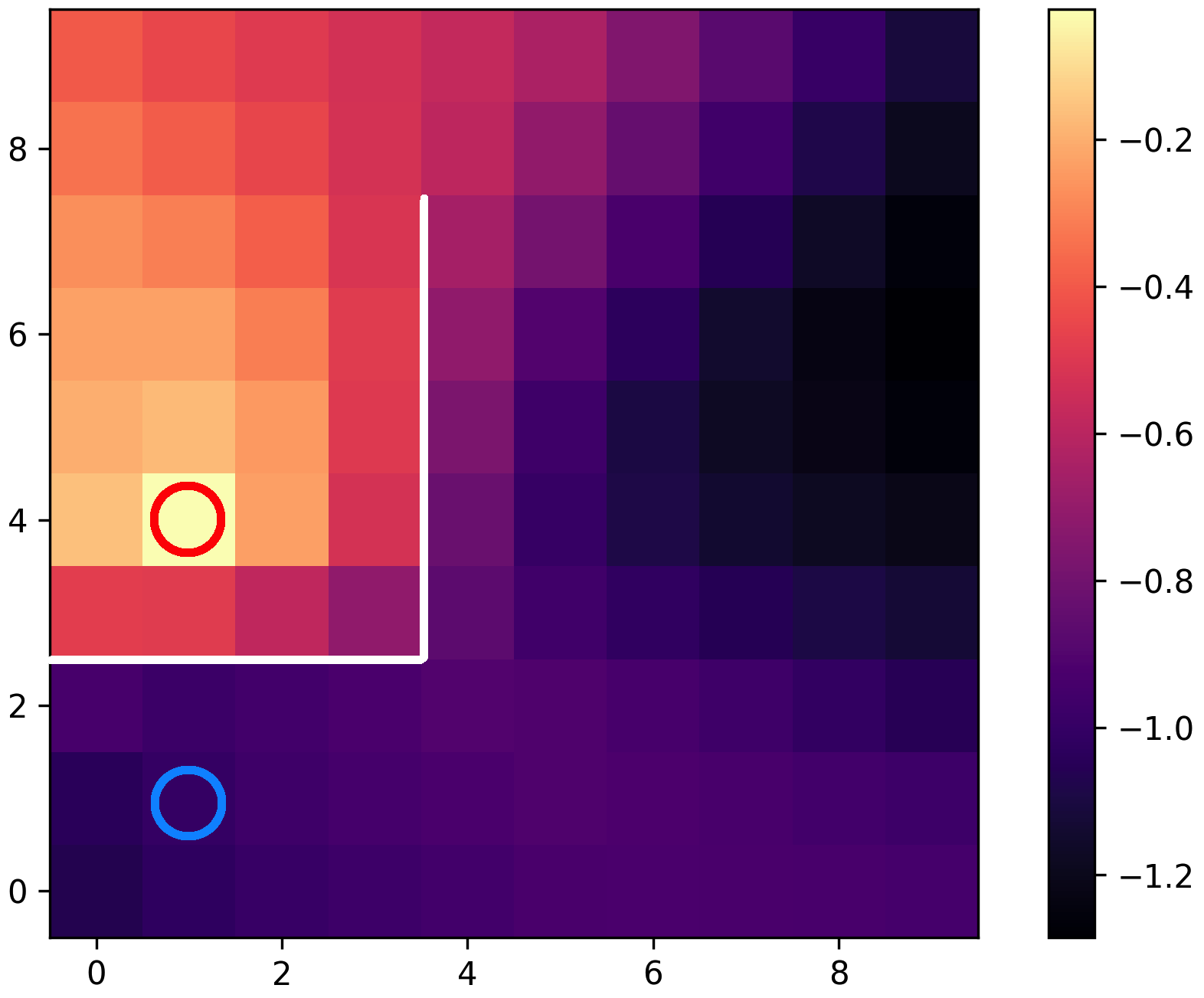}
        \caption{Learned Reward (100 training iterations)}
        \label{fig:reward}
    \end{subfigure}
    \hfill
    \begin{subfigure}[t]{.3\textwidth}
        \centering
        \includegraphics[width=\linewidth]{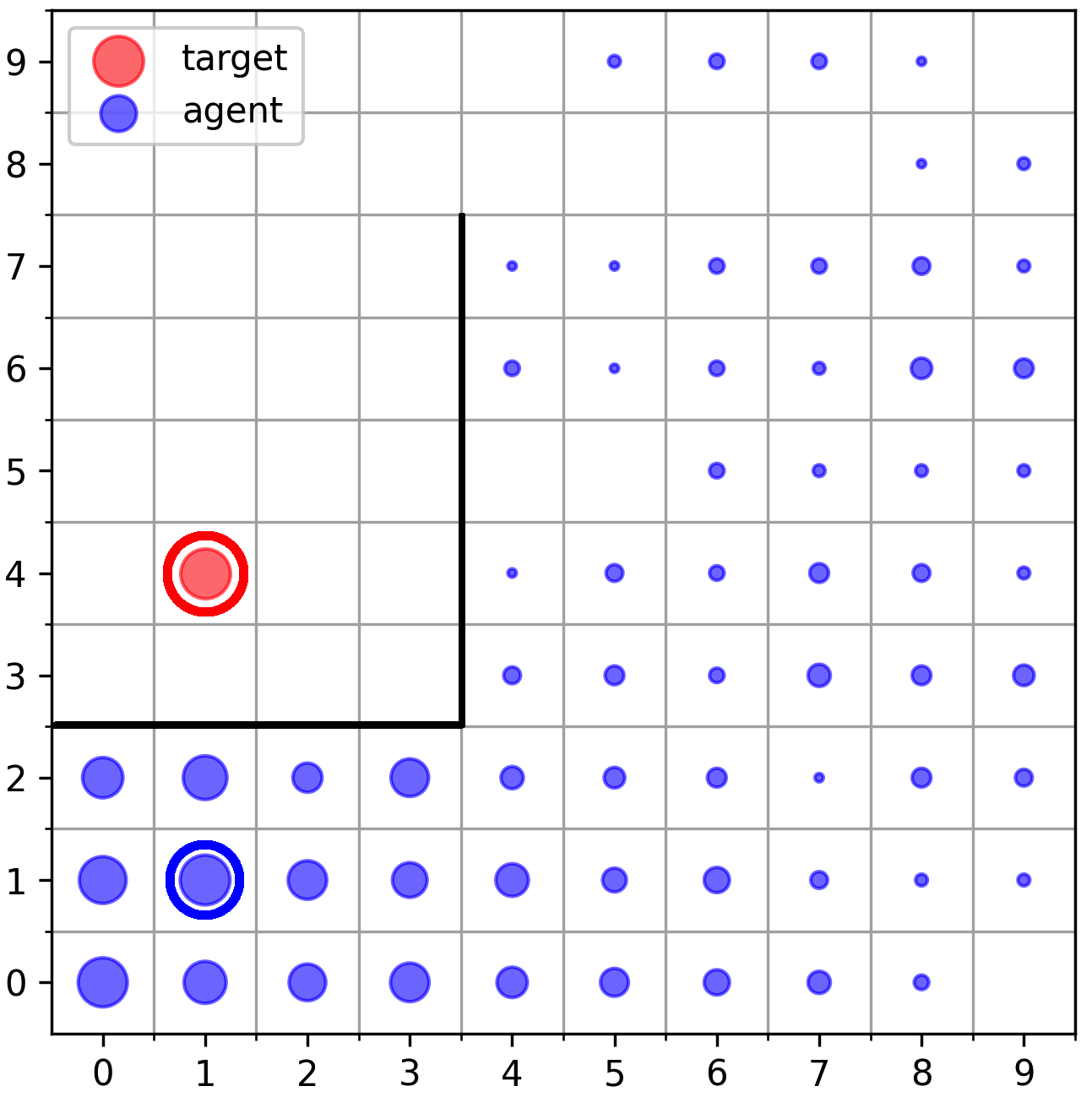}
        \caption{Agent state visitation learning with task reward (500 iterations)}
        \label{fig:task_policy}
    \end{subfigure}
    \caption{Grid world experiments. Agent's undiscounted state visitation (\ref{fig:grid_policy}, \ref{fig:task_policy}): Blue circle indicates agent's start state. Red circle is the goal. Blue bubbles indicate relative time agent's policy causes it to spend in respective states. Learned reward function (\ref{fig:reward}): \aim\ reward at each state of the grid world. Bold black (or white) lines indicate walls the agent cannot transition through.}
    \label{fig:gridworld}
\end{figure}

In Appendix \ref{app:grid} we also compare to the baselines described above and show that \aim\ learns a reward that is more efficient at directing the agent's exploration and more flexible to variations of the environment dynamics, such as stochastic dynamics or transitions that cause a sharp change in the state features.
None of the baselines compared to were able to direct the agent to the goal in this grid world even after given up to $5\times$ more interactions with the environment to train.
\aim's use of the time-step metric also enabled it to adapt to variations of the environment dynamics better than the gradient penalty based regularization used in Wasserstein GANs \citep{gulrajani_improved_2017} and adversarial imitation learning \citep{ghasemipour2020divergence} approaches.

\paragraph{Fetch Robot}
The generalization capability of \aim\ across multiple goals in goal-conditioned RL tasks with continuous states and actions is tested in the MuJoCo simulator \citep{todorov2012mujoco}, on the Fetch robot tasks from OpenAI gym \citep{brockman2016openai} which have been used to evaluate learning of goal-conditioned policies previously \citep{andrychowicz_hindsight_2018, zhang_automatic_2020}.
Descriptions of these tasks and their goal space is in Appendix \ref{app:fetch}.
We soften the Dirac target distribution for continuous states to instead be a Gaussian with variance of $0.01$ of the range of each feature.

The goals in this setting are not the full state, but rather the dimensions of factored states relevant to the given goal.
The task wrapper additionally returns the features of the agent's state in this reduced goal space, and so \aim\ can use it to learn our reward function, rather than the full state space.
It is unclear how this smaller goal space might affect \aim.
While the smaller goal space might make learning easier for potential function $f_\phi$, the partially observable nature of the goals might lead to a less informative reward.

We combine \aim\ with \textsc{her} (refer \autoref{sec:goal}) and refer to it as [\aim\ + \textsc{her}].
We compare this agent to the baselines we referred to above, as well as the sparse environment reward (\textsc{r} + \textsc{her}) and the dense reward derived from the negative Euclidean ($L2$) distance to the goal ($-L2$ + \textsc{her}).
The $L2$ distance is proportional to the number of steps it should take the agent to reach the goal in this environment, and so the reward based on it can act as an oracle reward that we can use to test how efficiently \aim\ can learn a reward function that helps the agent learn its policy.
We used the \textsc{her} implementation using Twin Delayed DDPG (\textsc{td3}) \citep{fujimoto2018TD3} as the underlying RL algorithm from the stable baselines repository \citep{stable-baselines}.
We did an extensive sweep of the hyperparameters for the baseline \textsc{her} + \textsc{r} (laid out in Appendix \ref{app:fetch}), with a coarser search on relevant hyperparameters for \aim.

\begin{figure}[t]

    \centering
    \begin{subfigure}{.48\textwidth}
    \includegraphics[width=\textwidth]{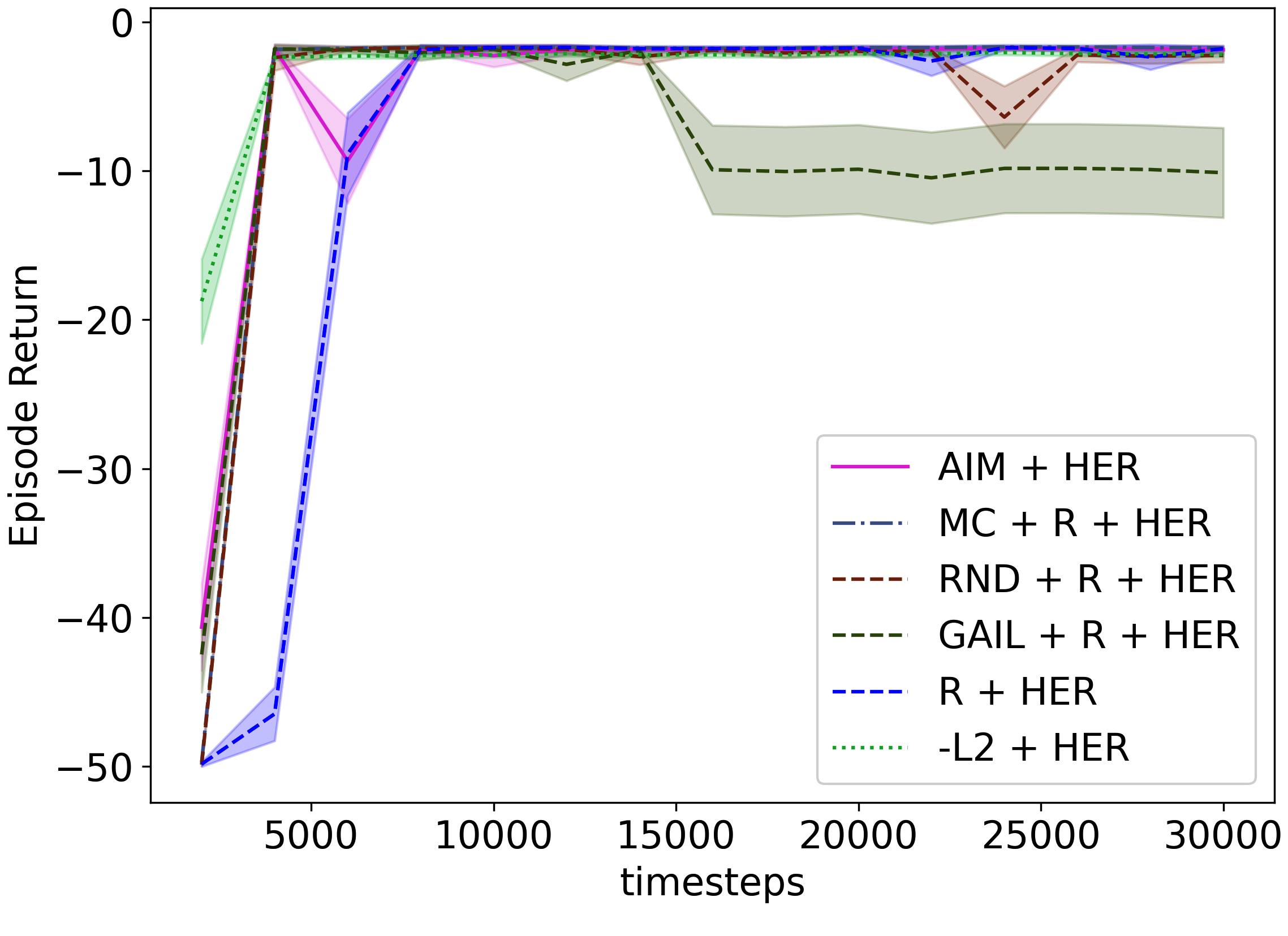} 
    \caption{Reach}
    \label{fig:reach}
    \end{subfigure}%
    \hfill
    \begin{subfigure}{.48\textwidth}
    \includegraphics[width=\textwidth]{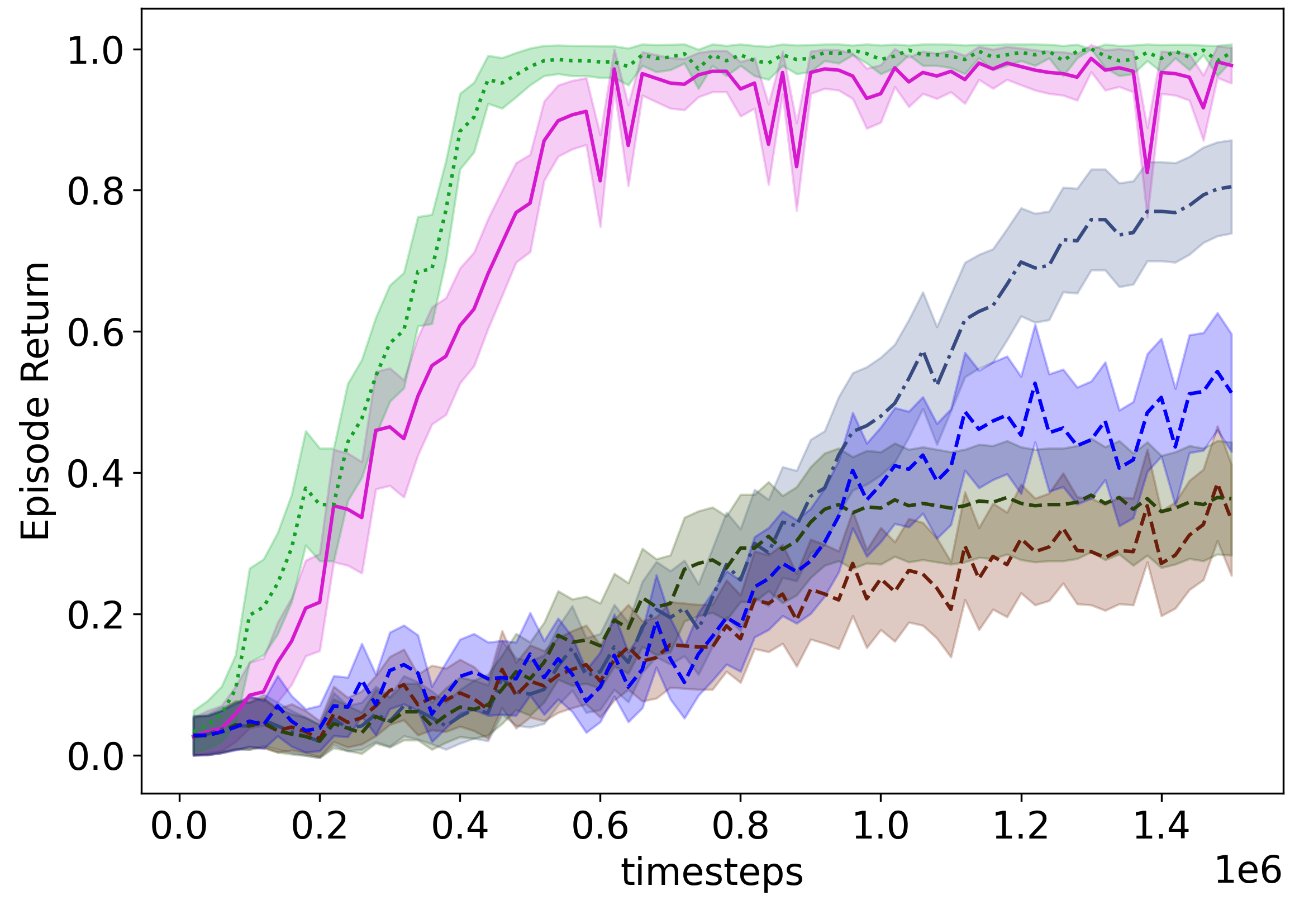}
    \caption{Pick and Place}
    \label{fig:pick}
    \end{subfigure}%
    \\
    \begin{subfigure}{.48\textwidth}
    \includegraphics[width=\textwidth]{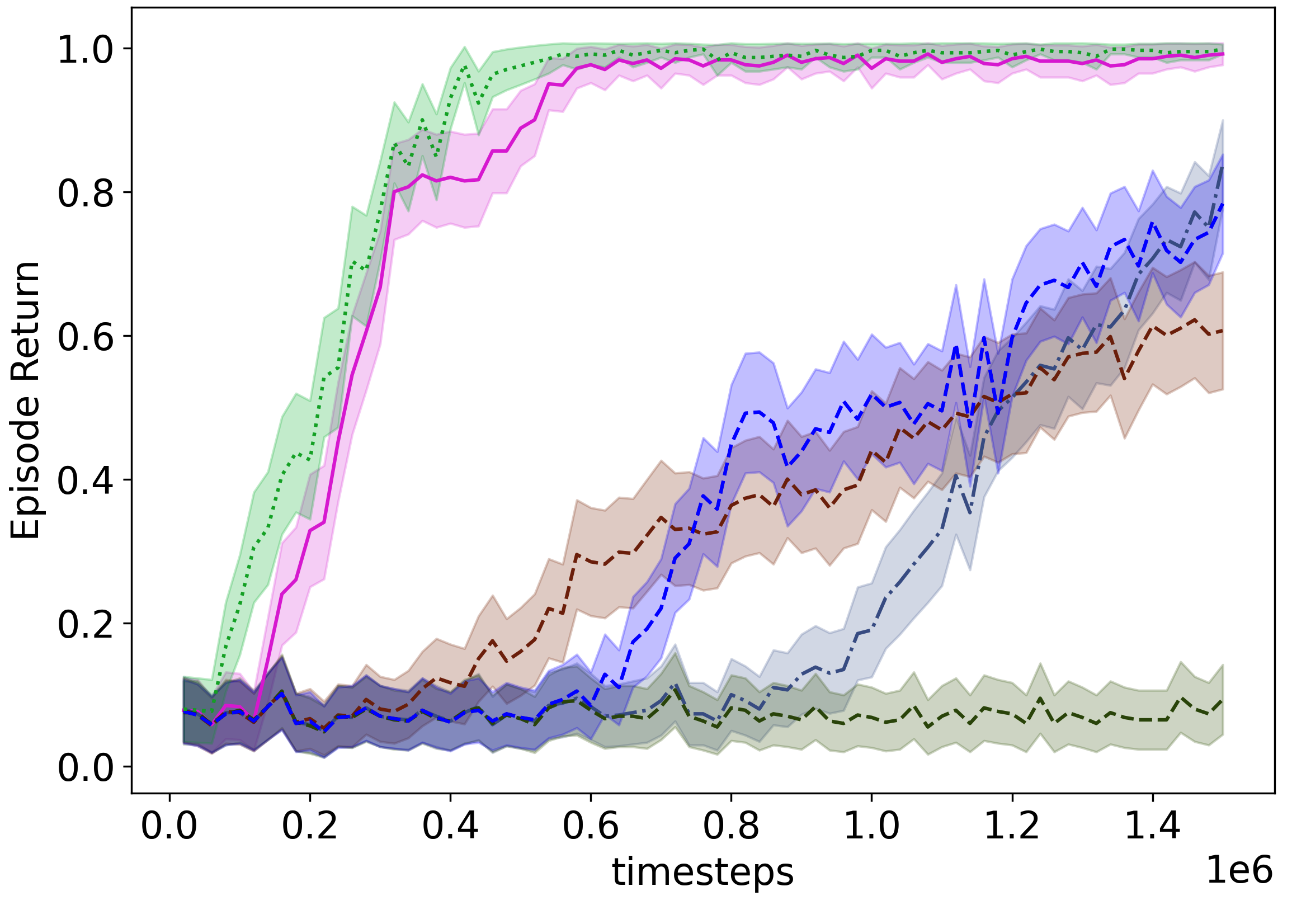}
    \caption{Push}
    \label{fig:push}
    \end{subfigure}%
    \hfill
    \begin{subfigure}{.48\textwidth}
    \includegraphics[width=\textwidth]{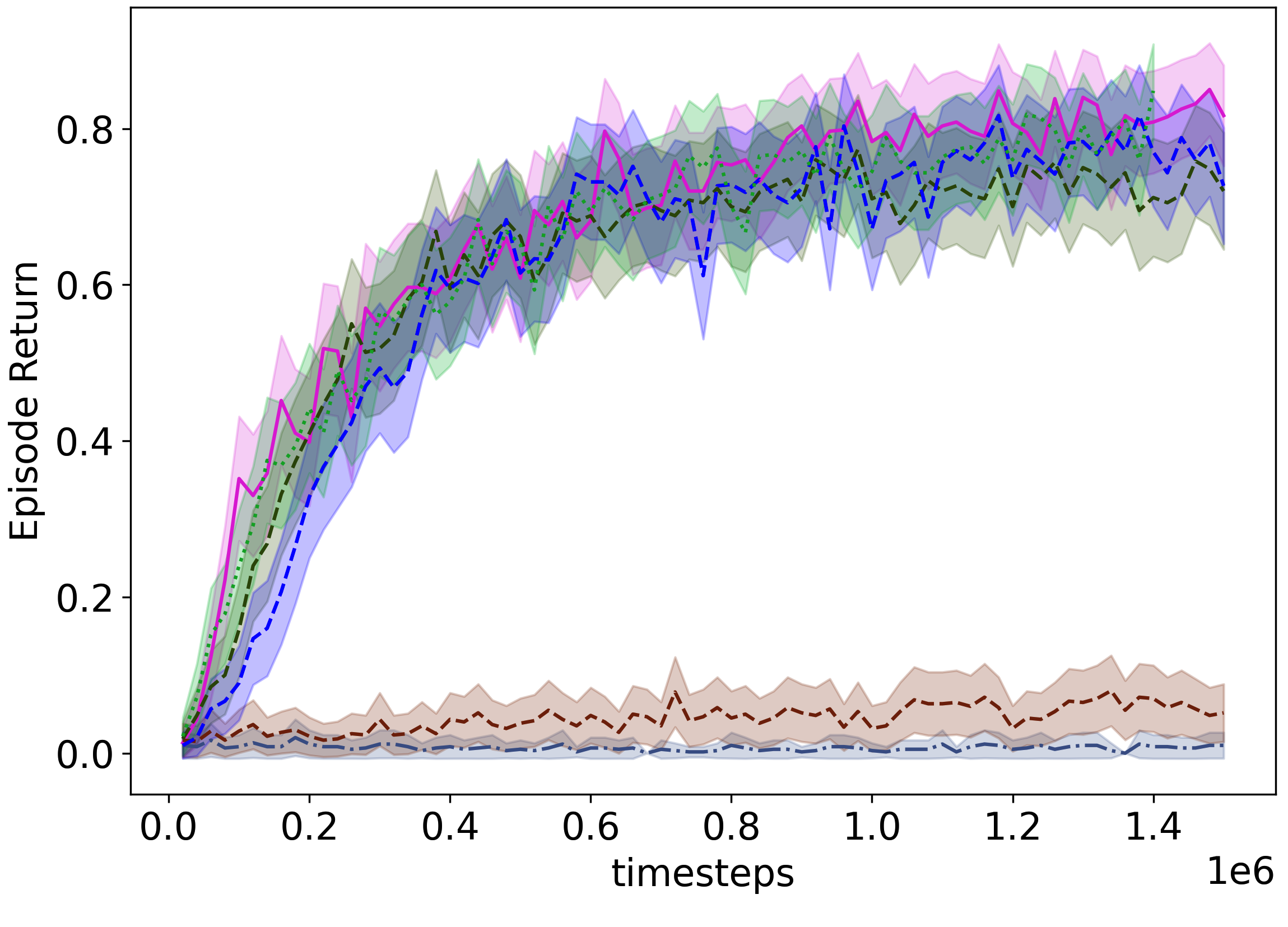}
    \caption{Slide}
    \label{fig:slide}
    \end{subfigure}
    \caption{Evaluating \textsc{aim} with \textsc{her} on some goal-conditioned RL tasks in the Fetch domain. \aim\ learns the reward function in tandem with the policy updates. The ``$-L2$'' reward is the true negative distance to goal, acting as a oracle reward in this domain. The other baselines are detailed above.}
    \label{fig:fetch}
\end{figure}

Figure \ref{fig:fetch} shows that using the \aim-learned reward speeds up learning in three of the four Fetch domains, even without the environment reward.
This improvement is very close to what we would see if we used the dense reward (based on the actual distance function in this domain).
An additional comparison with an agent learning with both the \aim-learned reward as well as the task reward (\aim\ + \textsc{r} + \textsc{her}) can be seen in \autoref{fig:app_fetch} in the Appendix, showing that using both signals accelerates learning even more.
These results also highlight that \aim\ continues to work in continuous state and action spaces, even though our analysis focuses on discrete states and actions.
Results are averaged across the 6 different seeds, with shaded regions showing standard error across runs.
Statistical analysis using a mixed effects ANOVA and a Tukey test at a significance level of $95\%$ (more detail in \autoref{app:fetch_test}) show that  in three of the four environments \aim\ and \aim + \textsc{r} have similar odds of reaching the goal as the dense shaped reward, and in all four environments \aim\ and \aim + \textsc{r} have higher odds of reaching the goal compared to the sparse reward.

The other baselines compare well to \aim\ in the Fetch Reach domain (Figure \ref{fig:reach}), but do not do as well on the other problems.
In fact, none of the other baselines outperform the vanilla baseline [\textsc{r} + \textsc{her}] in all the domains.
The \textsc{rnd} rewards help the agent to start learning faster in the Push domain (Figure \ref{fig:push}), but lead to worse performance in Pick and Place (Figure \ref{fig:pick}).
On the other hand, learning the distance function through \textsc{mc} regression helps in the Pick and Place domain, but slows down learning when dealing with Push.
Most notably, both these approaches cause learning to fail in the Slide domain (Figure \ref{fig:slide}), where the credit assignment problem is especially difficult.
\textsc{gail} works as well as \aim\ and the vanilla baseline in Slide, but underperforms in the other domains.
We hypothesize that the additional rewards in these baselines conflict with the task reward.
Additionally, none of the three new baselines work well if we do not provide the task reward in addition to the specific bonus for that algorithm. 

We did not find any configuration in the Fetch Reach domain where [\textsc{sm}i\textsc{rl} + \textsc{r} + \textsc{her}] was able to accomplish the task in the given training budget.
Since SMiRL did not work on the grid world or Fetch Reach, we did not try it out on any of the other domains.

\textsc{fairl} \citep{ghasemipour2020divergence} (which has been shown to learn policies that cover hand-specified state distributions) was also applied on these $4$ domains but it failed to learn at all.
Interestingly, scaling the reward such that it is always negative led to similar performance to (but not better than) \aim.
We hypothesize that \textsc{fairl}, as defined and presented, fails in these domains because the environments are episodic, and the episode ends earlier if the goal is reached.
Since the \textsc{fairl} reward is positive closer to the target distribution, the agent can get close to the target, but refrain from reaching it (and ending the episode) to collect additional positive reward.

The domain where \aim\ does not seem to have a large advantage (Slide) is one where the agent strikes an object initially and that object has to come to rest near the goal.
In fact, \aim-learned rewards, the vanilla environment reward R, and the oracle $-L2$ rewards all lead to similar learning behavior, indicating that this particular task does not benefit much from shaped rewards.
The reason for this invariance might be that credit assignment has to propagate back to the single point when the agent strikes the object regardless of how dense the subsequent reward is.

\section{Discussion and Future Work} \label{sec:disc}

Approaches for estimating the Wasserstein distance to a target distribution by considering the dual of the Kantorovich relaxation have been previously proposed \citep{arjovsky_wasserstein_2017,gulrajani_improved_2017,xiao_wasserstein_2019}, but assume that the ground metric is the $L2$ distance.
We improve upon them by choosing a metric space more suited to the MDP and notions of optimality in the MDP.
This choice allows us to leverage the structure introduced by the dynamics of the MDP to regularize the Kantorovich potential using a novel objective.

Previous work \cite{Bellemare2017CramerD} has pointed out that the gradients from sample estimates of the Wasserstein distance might be biased.
This issue is mitigated in our implementation through multiple updates of the discriminator, which they found to be empirically useful in reducing the bias.
Additionally, recent work has pointed out that the discriminator in WGAN might be bad at estimating the Wasserstein distance \citep{Stanczuk2021WassersteinGW}.
While our experiments indicate that the potential function in \aim\ is learned appropriately, future work could look more deeply to verify possible inefficiencies in this estimation.

The process of learning the Wasserstein distance through samples of the environment while simultaneously estimating the cost of the full path is reminiscent of the $A^*$ algorithm \citep{hart1968formal}, where the optimistic heuristic encourages the agent to explore in a directed manner, and adjusts its estimates based on these explorations.

The discriminator objective (Equation \ref{eqn:f_loss}) also bears some resemblance to a linear program formulation of the RL problem \citep{puterman1990markov}.
The difference is that this formulation minimizes the value function on states visited by the agent, while \aim\ additionally maximizes the potential at the goal state.
This crucial difference has two main consequences.
First, the potential function during learning is not equivalent to the value of the agent’s policy (verified by using this potential as a critic).
Second, increasing the potential of the goal state in \aim\ directs the agent exploration in a particular direction (namely, the direction of sharpest increase in potential).

In the goal-conditioned RL setting, \aim\ seems to be an effective intrinsic reward that balances exploration and exploitation for the task at hand.
The next step is to consider whether the Wasserstein distance can be estimated similarly for more general tasks, and whether minimizing this distance in those tasks leads to the optimal policy.
A different potential avenue for future work is the problem of more general exploration \citep{hazan2019provably,lee2019efficient} by specifying a uniform distribution as the target, or using this directed exploration as an intermediate step for efficient exploration \citep{Jinnai2020Exploration}.

Finally, reward design is an important aspect of practical reinforcement learning.
Not only do properly shaped reward speed up learning \citep{ng1999policy}, but reward design can also subtly influence the kinds of behaviors deemed acceptable for the RL agent \citep{knox2021reward} and could be a potential safety issue keeping reinforcement learning from being deployed on real world problems.
Learning-based approaches that can assist in specifying reward functions safely given alternative approaches for communicating the task could be of value in such a process of reward design, and an avenue for future research.
\section*{Acknowledgements and Funding Information}

We thank Caroline Wang, Garrett Warnell, and Elad Liebman for discussion and feedback on this work.
We also thank the reviewers for their thoughtful comments and suggestions that have helped to improve this paper.

This work has taken place in part in the Learning Agents Research
Group (LARG) at the Artificial Intelligence Laboratory, and in part in the Personal Autonomous Robotics Lab (PeARL) at The University
of Texas at Austin.
LARG research is supported in part by the
National Science Foundation (CPS-1739964, IIS-1724157, FAIN-2019844),
the Office of Naval Research (N00014-18-2243), Army Research Office
(W911NF-19-2-0333), DARPA, Lockheed Martin, General Motors, Bosch, and
Good Systems, a research grand challenge at the University of Texas at
Austin.
PeARL research is supported in part by the NSF (IIS-1724157, IIS-1638107, IIS-1749204, IIS-1925082), ONR (N00014-18-2243), AFOSR (FA9550-20-1-0077), and ARO (78372-CS). This research was also sponsored by the Army Research Office under Cooperative Agreement Number W911NF-19-2-0333. 
The views and conclusions contained in this document are
those of the authors and should not be interpreted as representing the official policies, either expressed or implied, of the Army Research Office or the U.S. Government.
The U.S. Government is authorized to reproduce and distribute reprints for Government purposes notwithstanding any copyright notation herein.
Peter Stone serves as the Executive
Director of Sony AI America and receives financial compensation for
this work. The terms of this arrangement have been reviewed and
approved by the University of Texas at Austin in accordance with its
policy on objectivity in research.

\bibliographystyle{plainnat}
\bibliography{main}
\clearpage

\clearpage
\appendix

\section{Metrics and Quasimetrics} \label{sec:metrics}

A metric space ($\mathcal{M}, d$) is composed of a set $\mathcal{M}$ and a metric $d: \mathcal{M} \times \mathcal{M} \longmapsto \mathbb{R}^+ \cup \{\infty\}$ that compares two points in that set.
Here $\mathbb{R}^+$ is the set of non-negative real numbers.

\begin{definition}
A metric $d: \mathcal{M} \times \mathcal{M} \longmapsto \mathbb{R}^+ \cup \{\infty\}$ compares two points in set $\mathcal{M}$ and satisfies the following axioms $\forall m_1, m_2, m_3 \in \mathcal{M}$:
\begin{itemize}
    \item $d(m_1, m_2) = 0 \iff m_1 = m_2$ (identity of indiscernibles)
    \item $d(m_1, m_2) = d(m_2, m_1)$ (symmetry)
    \item $d(m_1, m_2) \leq d(m_1, m_3) + d(m_3, m_2)$ (triangle inequality)
\end{itemize}
\end{definition}

A variation on metrics that is important to this paper is \emph{quasimetrics}.

\begin{definition}
A quasimetric \citep{Smyth1987QuasiUR} is a function that satisfies all the properties of a metric, with the exception of symmetry $d(m_1, m_2) \neq d(m_2, m_1)$.
\end{definition}

As an example, consider an MDP where the actions and transition dynamics allow an agent to navigate from any state to any other state.
Let $T(s_2| \pi, s_1)$ be the random variable for the first time-step that state $s_2$ is encountered by the agent after starting in state $s_1$ and following policy $\pi$.
The time-step metric $d^\pi_T$ for this MDP can then be defined as

\begin{align*}
    d^\pi_T(s_1, s_2) \defd \mathbb{E}\; \left[T(s_2 |\pi, s_1)\right]
\end{align*}

$d^\pi_T$ is a quasimetric, since the action space and transition function need not be symmetric, meaning the expected minimum time needed to go from $s_1$ to $s_2$ need not be the same as the expected minimum time needed to from $s_2$ to $s_1$.
The diameter of an MDP \citep{jaksch_near-optimal_2010,kearns2002near} is generally calculated by taking the maximum time-step distance between over all pairs of states in the MDP either under a random policy or a policy that travels from any state to any other state in as few steps as possible.

\section{Optimal Transport and Wasserstein-1 Distance}
\label{app:opt}

The theory of optimal transport \citep{villani2008optimal,bousquet_optimal_2017} considers the question of how much work must be done to transport one distribution to another optimally.
More concretely, suppose we have a metric space ($\mathcal{M}$, $d$) where $\mathcal{M}$ is a set and $d$ is a metric on $\mathcal{M}$.
See the definitions of metrics and quasimetrics in Appendix \ref{sec:metrics}.
For two distributions $\mu$ and $\nu$ with finite moments on the set $\mathcal{M}$, the Wasserstein-$p$ distance is denoted by: 

\begin{align} \label{appeqn:w-p}
    W_p(\mu, \nu) \defd \inf_{\zeta \in Z(\mu, \nu)} \E_{(X,Y)\sim \zeta} \left[d(X, Y)^p\right]^{1/p}
\end{align}

where $Z$ is the space of all possible couplings between $\mu$ and $\nu$.
Put another way, $Z$ is the space of all possible distributions $\zeta \in \Delta(\mathcal{M}\times\mathcal{M})$ whose marginals are $\mu$ and $\nu$ respectively.
Finding this optimal coupling tells us what is the least amount of work, as measured by $d$, that needs to be done to convert $\mu$ to $\nu$.
This Wasserstein-$p$ distance can then be used as a cost function (negative reward) by an RL agent to match a given target distribution \citep{xiao_wasserstein_2019,dadashi_primal_2020}.

Finding the ideal coupling (meaning finding the optimal transport plan from one distribution to the other) which gives us an accurate distance is generally considered intractable.
However, if what we need is an accurate estimate of the Wasserstein distance and not the optimal transport plan (as is the case when we mean to use this distance as part of our intrinsic reward) we can turn our attention to the dual form of this distance.
The Kantorovich-Rubinstein duality \citep{villani2008optimal} for the Wasserstein-1 distance on a ground metric $d$ is of particular interest and gives us the following equality:

\begin{align} \label{appeqn:KRdual}
    W_1(\mu, \nu) = \sup_{\text{Lip}(f) \leq 1} \mathbb{E}_{y \sim \nu}\left[f(y)\right] - \mathbb{E}_{x \sim \mu} \left[f(x)\right]
\end{align}

where the supremum is over all $1$-Lipschitz functions $f: \mathcal{M} \longmapsto \mathbb{R}$ in the metric space, and the Lipschitz constant of a function $f$ is defined as:
\begin{align} \label{eqn:lip}
    \text{Lip}(f) &\defd \sup \left\{ \frac{|f(y) - f(x)|}{d(x, y)} \forall (x, y) \in \mathcal{M}^2, x \neq y \right\}
\end{align}

That is, the Lipschitz condition of this function $f$ (called the Kantorovich potential function) is measured according to the metric $d$.
Recently, \citet{jevtic2018combinatorial} has shown that this dual formulation where the constraint on the potential function is a smoothness constraint extends to quasimetric spaces as well.
If defined over a quasimetric space, the Wasserstein distance also has properties of a quasimetric (specifically, the distances are not necessarily symmetric).

If the given metric space is a Euclidean space ($d(x, y) = \|y - x\|_2$), the Lipschitz bound in Equation \ref{eqn:KRdual} can be computed locally as a uniform bound on the gradient of $f$.

\begin{align} \label{eqn:KRdual_l2}
    W_1(\mu, \nu) = \sup_{\|\nabla f\| \leq 1} \mathbb{E}_{y \sim \nu}\left[f(y)\right] - \mathbb{E}_{x \sim \mu} \left[f(x)\right]
\end{align}

meaning that $f$ is the solution to an optimization objective with the restriction that $\|\nabla f(x)\| \leq 1$ for all $x \in \mathcal{M}$.
This strong bound on the dual in Euclidean space is the one that has been used most in recent implementations of the Wasserstein generative adversarial network \citep{arjovsky_wasserstein_2017,gulrajani_improved_2017} to regularize the learning of the discriminator function.
Such regularization has been found to be effective for stability in other adversarial learning approaches such as adversarial imitation learning \citep{ghasemipour2020divergence}.

Practically, the Kantorovich potential function $f$ can be approximated using samples from the two distributions $\mu$ and $\nu$, regularization of the potential function to ensure smoothness, and an expressive function approximator such as a neural network.
A more in depth treatment of the Kantorovich relaxation and the Kantorovich-Rubinstein duality, as well as their application in metric and Euclidean spaces using the Wasserstein-1 distance we lay out above, is provided by \citet{peyre_computational_2020}.

Now consider the problem of goal-conditioned reinforcement learning.
Here the target distribution $\nu$ is the goal-conditioned target distribution $\rho_g$ which is a Dirac at the given goal state.
Similarly, the distribution to be transported $\mu$ is the agent's goal-conditioned state distribution $\rho_\pi$.

The Wasserstein-1 distance of an agent executing policy $\pi$ to the goal $s_g$ can be expressed in a fairly straightforward manner as:
\begin{align} \label{eqn:goal_W1}
    W_1(\rho_\pi, \rho_g) = \sum_{s \in \sset} \rho_\pi(s | s_g) d(s, s_g)
\end{align}

The above is a simplification of Equation \ref{eqn:w-p}, where $p=1$ and the joint distribution is easy to specify since the target distribution $\rho_g$ is a Dirac distribution.

\section{Lipschitz constant of Potential function} \label{app:quasi_lip}

For a given goal $s_g$ and all states $s_0 \in \sset$, recall that function $f$ is $L$-Lipschitz if it follows the Lipschitz condition as follows.

\begin{align}
    \lvert f(s_g) - f(s_0)\rvert \leq L d^\pi_T(s_0, s_g)\; \forall s_0 \in \sset
\end{align}

\begin{restatable}{proposition}{lipschitz_cond}
\label{prop:lipschitz_cond}
If transitions from the agent policy $\pi$ are guaranteed to arrive at the goal in finite time and $f$ is $L$-bounded in expected transitions, i.e.,
\begin{align*}
\sup_{s\in S} \mathop{\E}_{s'\sim \pi, P} \left[\lvert f(s') - f(s)\rvert \right] \leq L,
\end{align*}
then $f$ is $L$-Lipschitz.
\end{restatable}
\begin{proof}
Since $f(s_g) - f(s_0)$ is a scalar quantity, we may write $f(s_g) - f(s_0) = \E_{\pi, P}[f(s_g) - f(s_0)]$. Using this fact and that $P(T(s_0) < \infty)=1$ where $T(s_0)=T^\pi(s_g | \pi, s_0)$ for notation simplicity, the LHS of the expression above becomes a telescopic sum
\begin{align*}
\lvert f(s_g) - f(s_0)\rvert &= \mathop{\E}_{\pi, P} \left[f(s_g) - f(s_0)\right] \\
    &= \mathop{\E}_{\pi, P} \left[\left\lvert \sum_{t=0}^{T(s_0) - 1} (f(s_{t+1}) - f(s_t))\right\rvert \right].
    \\
    &\leq \mathop{\E}_{\pi, P} \left[ \sum_{t=0}^{T(s_0) - 1} \lvert f(s_{t+1}) - f(s_t)\rvert \right]. \\
\end{align*}
Now let us assume that for all transitions $(s, a, s')$, $\E[\lvert f(s') - f(s)\rvert] \leq L$. Then
\begin{align*}
   \mathop{\E}_{\pi, P} \left[ \sum_{t=0}^{T(s_0) - 1} \vert f(s_{t+1}) - f(s_t)\rvert\right]&= \mathop{\E}_{T(s_0)}\left[\mathop{\E}_{\pi, P} \left[ \sum_{t=0}^{T(s_0) - 1} \lvert f(s_{t+1}) - f(s_t)\rvert \Big\vert T(s_0) \right] \right] \\
    &\leq \mathop{E}_{T(s_0)} \left[ \sum_{t=0}^{T(s_0) - 1} L \right] \\
    &= L\mathop{\E}_{T(s_0)}\left[ T(s_0) \right] \\
    &= L d_T^\pi(s_0, s_g),
\end{align*}
showing that $|f(s_g) - f(s_0)| \leq L d_T^\pi(s_0, s_g)$ as desired. 
\end{proof}

\section{Proofs of Claims} \label{app:proofs}

The Bellman optimality condition gives us the following optimal distance to goal:
\begin{align} \label{eqn:opt_dt}
    d^{\blacklozenge}_T(s, s_g) = \begin{cases} 0 & \text{if } s = s_g \\
    1 + \min_{a \in \aset} \sum_{s' \in \sset} P(s'|s, a, s_g) d^{\blacklozenge}_T(s', s_g) & \text{otherwise}\end{cases}
\end{align}

\lowerbound*
\begin{proof}
\begin{align*}
    V^\pi(s|s_g) &= \mathbb{E} \left[ \gamma^{T(s_g|\pi, s)} \right] \geq \gamma^{d^\pi_T(s, s_g)} \quad \forall \ s \in \sset
\end{align*}
where the inequality follows as a consequence of Jensen's inequality and the convex nature of the value function.
\end{proof}

\policyCorrespondence*

\begin{proof}
Consider the value of a state $s$ given goal $s_g$.
If the transitions are deterministic and the agent policy $\pi$ is deterministic (as is the case for the optimal policy), then the time to reach the goal satisfies $\mathrm{Var}(T(s_g|\pi, s))=0$, implying that $\Delta_{\text{Jensen}}$ vanishes and therefore
\begin{align*}
    V^\pi(s|s_g) &= \gamma^{d^{\pi}_T(s, s_g)}.
\end{align*}
Since $\gamma \in [0, 1)$, $V^\pi$ is monotonically decreasing with $d_T^\pi$
\begin{align*}
    \argmax_{\pi}V^\pi(s|s_g) = \argmin_\pi d_T^\pi(s, s_g) \; \forall \ s \in \sset
\end{align*}

That is, in the deterministic transition dynamics scenario, $\pi^* = \pi^\blacklozenge$.
\end{proof}

\analytical*
\begin{proof}   
The first step of the proof is to obtain an analytical expression for the the expected distance to the goal after $t$ steps as a function of the expected distance at $t=0$. To reduce the notation burden, denote $T(s_0)=T(s_g|\pi, s_0)$ and let $s_t(s_0)$ be the state after $t$ steps conditional on some starting state $s_0$ where actions are taken according to $\pi$. We have excluded $s_g$ and $\pi$ from the notation since they are fixed for the purpose of this proposition. Using the law of total expectation we have that for every initial $s_0$
\begin{align*}
\E_{s_t}[d(s_t(s_0), s_g)] &= \E_{T(s_0)}[\E_{s_t}[d(s_t(s_0), s_g) \mid T(s_0)]] = \E_{T(s_0)}[\max(T(s_0) - t, 0)],
\end{align*}
Now, by expanding the definition of $\rho_\pi(s\mid s_g)$ in \eqref{eqn:dt_W1}, exchanging the order of summation, and using the previous equation we may write
\begin{equation*}
\begin{aligned}
W_1^\pi(\rho_\pi, \rho_g) &=  \sum_{s \in \sset}  \sum_{t=0}^\infty (1 - \gamma) \gamma^t \E_{s_0}[P(s_t=s \mid \pi, s_g)] d^\pi_T(s, s_g)\\
&=\E_{s_0}\left[ (1 - \gamma) \sum_{t=0}^\infty \gamma^t \E_{s_t}[d(s_t(s_0), s_g) \mid s_0]\right] \\
&=  \E_{s_0}\left[\E_{T(s_0)}\left[(1 - \gamma) \sum_{t=0}^\infty \gamma^t \max(T(s_0) - t, 0) \Big\vert s_0\right]\right] 
\end{aligned}
\end{equation*}

Standard but tedious algebraic manipulations given in Lemma \ref{prop:analytical:support1} in the Appendix show that
\begin{align*}
\sum_{t=0}^\infty (1 - \gamma) \gamma^t \max(T(s_0) - t, 0) = T(s_0) - \frac{\gamma}{1-\gamma}(1 - \gamma^{T(s_0)}).
\end{align*}
Combining the two identities above we arrive at
\begin{equation}\label{eqn:analytical}
\begin{aligned}
W_1^\pi(\rho_\pi, \rho_g)  &= \E_{s_0}\left[\E_{T(s_0)}\left[T(s_0) - \frac{\gamma}{1 - \gamma}(1 - \gamma^{T(s_0)}) \Big\vert s_0\right]\right] \\&= \E_{s_0}\left[d(s_0, s_g) - \frac{\gamma}{1 - \gamma}(1 - \E[\gamma^{T(s_0)} \mid s_0])\right] \\
 & =\E_{s_0}\left[d(s_0, s_g) + \frac{\gamma}{1-\gamma}  \gamma^{d(s_0,s_g)}  - \frac{\gamma}{1 - \gamma}(1 - \E[\gamma^{T(s_0)}\mid s_0] + \gamma^{d(s_0,s_g)} ) \right]\\
 & =\E_{s_0}\left[d(s_0, s_g) + \frac{\gamma}{1-\gamma}  \gamma^{d(s_0,s_g)}  + \frac{\gamma}{1 - \gamma}(\Delta^\pi_{\text{Jensen}}(s_0) - 1)\right].
\end{aligned}
\end{equation}

To finalize the proof, we only need to show that the function $h(\mu) = \mu + \frac{\gamma}{1 - \gamma}\gamma^\mu$ is monotonically increasing for every $\gamma\in [0, 1)$. This is a standard calculus exercise that we show in Lemma \ref{prop:analytical:support2} in Appendix \ref{appendix:aux-prop3}.
\end{proof}

\aimsearch*

\begin{proof}
Proposition \ref{prop:corr} shows that the Jensen gap vanishes for the optimal policy of an MDP with deterministic transitions and that it minimizes the expected distance from start for all initial states.
Proposition \ref{prop:analytical}, on the other hand, implies that when the Jensen gap vanishes, the Wasserstein distance is monotonically increasing in the expected distance from the start.
Together, the two propositions show that $\pi^*$ minimizes the Wasserstein distance.
\end{proof}

\begin{algorithm2e}[t]
\caption{\textsc{aim} + \textsc{her}}
\label{alg}
\SetAlgoLined
\KwIn{Agent policy $\pi_{\theta}$, discriminator $f_{\phi}$, environment $env$, \\ number of Epochs $N$, number of time-steps per epoch $K$, \\ policy update period $k$, discriminator update period $m$, episode length $T$, \\ replay buffer (for HER), smaller replay buffer (for discriminator)}
Initialize discriminator parameters $\phi$\;
Initialize policy parameters $\theta$\;
\For{$n = 0, 1, \ldots , N - 1$}{
$t=0$\;
goal\_reached = True\;
\While{$t < K$}{
\If{goal\_reached or episode\_over}{
Sample goal $s_g \sim \sigma(\mathcal{G})$\;
Sample start state $s \sim \rho_0(\sset)$\;
goal\_reached = False\;
episode\_over = False\;
$t_{start} = K$\;
}
Sample action $a \sim \pi_\theta(\cdot|s, s_g)$\;
$s' = env.step(a)$\;
\If{$s' = s_g$}{
    goal\_reached = True\;
}
\tcp{end episode if goal not reached in $T$ steps}
\If{$t - t_{start} = T$}{
    episode\_over = True\;
}

Add $(s, a, s', s_g, goal\_reached)$ to replay buffer and smaller replay buffer\;
\If{goal\_reached or episode\_over}{ \label{alg:her}
    Add hindsight goals to both buffers\;
}
\tcp{Update policy parameters $\theta$ every $k$ steps}
\If{$t \% k = 0$}{
Sample tuples $(s, a, s', s_g, goal\_reached)$ from replay buffer\;
Get intrinsic reward (Equation \ref{eqn:reward})\;
Update policy parameters $\theta$ using any off-policy learning algorithm\;
}
\tcp{Update discriminator parameters $\phi$ every $m$ steps}
\If{$t \% m = 0$}{
Sample tuples $(s, a, s', s_g, goal\_reached)$ from smaller replay buffer\;
Update discriminator parameters $\phi$ using Equation \ref{eqn:f_loss}\;
}
$t = t + 1$\;
}
Evaluate agent policy\;
}
\end{algorithm2e}

\section{Auxiliary results for Proposition \ref{prop:analytical}}\label{appendix:aux-prop3}

\begin{lemma}\label{prop:analytical:support1} Let $T$ be a positive integer. Then
\begin{align*}
\sum_{t=0}^\infty (1 - \gamma) \gamma^t \max(T - t, 0) = T - \frac{\gamma}{1-\gamma}(1 - \gamma^{T}).
\end{align*}
\end{lemma}
\begin{proof}
Direct computation gives
\begin{align*}
 (1 - \gamma) \sum_{t=0}^\infty\gamma^t \max(T - t, 0) & = (1 - \gamma) \sum_{t=0}^{T -1} \gamma^t (T - t) \\
& = (1 - \gamma)T \sum_{t=0}^{T -1} \gamma^t - (1 - \gamma)\sum_{t=0}^{T -1} t \gamma^t
\end{align*}
We will now simplify the two terms of the last expression. For the first one, have
\begin{align*}
(1 - \gamma)T \sum_{t=0}^{T -1} \gamma^t  = (1 - \gamma)T\frac{1 - \gamma^T}{1 - \gamma} = T - T\gamma^T.
\end{align*}
For the second one, the computations are a bit more involved
\begin{align*}
(1 - \gamma)\sum_{t=0}^{T -1} t \gamma^t &= (1 - \gamma)\gamma \sum_{t=1}^{T -1} t \gamma^{t - 1} \\
& =(1 - \gamma) \sum_{t=1}^{T -1} \gamma \frac{d}{d\gamma}\gamma^{t} \\
& = \gamma(1 - \gamma) \frac{d}{d\gamma}\sum_{t=0}^{T -1} \gamma^{t}\\
& = \gamma(1 - \gamma) \frac{d}{d\gamma} \frac{1 - \gamma^T}{1 - \gamma} \\
& = \frac{\gamma}{(1 - \gamma)}\left(- T \gamma^{T - 1}(1 - \gamma) + (1  - \gamma^T)\right) = - T \gamma^T + \frac{\gamma}{(1 - \gamma)}(1 - \gamma^T).
\end{align*}
When combining the two simplified expressions the terms with $T \gamma^T$ will cancel out, yielding the desired expression.
\end{proof}

\begin{lemma} \label{prop:analytical:support2}  The function $h_\gamma(\mu) = \mu + \frac{\gamma}{1 - \gamma}\gamma^\mu$ is monotonically increasing for every $\gamma\in [0, 1)$.
\end{lemma}
\begin{proof}
We must show that $\frac{d}{d\mu}h_\gamma(\mu) > 0$ for every $\gamma\in[0,1)$ and every $\mu > 0$. Computing the derivative directly we obtain
\begin{align*}
\frac{d}{d\mu}h_\gamma(\mu) = 1 + \frac{\log(\gamma)\gamma^{\mu + 1}}{1 - \gamma}.
\end{align*}
Thus, it will suffice to show that the second term above is greater than -1. For this purpose, first note that $\log(\gamma)\gamma^{\mu + 1} > \log(\gamma)$ since $\gamma < 1$. Now, we use the fact that $\log(\gamma) < 1 - \gamma$ for $\gamma < 1$. This can be verified noting that $1 - \gamma$ is the tangent line to the concave curve $\log(\gamma)$ and the curves meet at $\gamma=1$. And therefore $\log(\gamma) / (1 - \gamma) > -1$. Putting these observation together,
\begin{align*}
\frac{d}{d\mu}h_\gamma(\mu)  = 1 + \frac{\log(\gamma)\gamma^{\mu + 1}}{1 - \gamma} > 1 + \frac{\log(\gamma)}{1 - \gamma} > 1 - 1 = 0,
\end{align*}
concluding the proof.
\end{proof}

\section{Grid World Experiments} \label{app:grid}

\begin{figure}
    \centering
    \begin{subfigure}[t]{.3\textwidth}
        \centering
        \includegraphics[width=\linewidth]{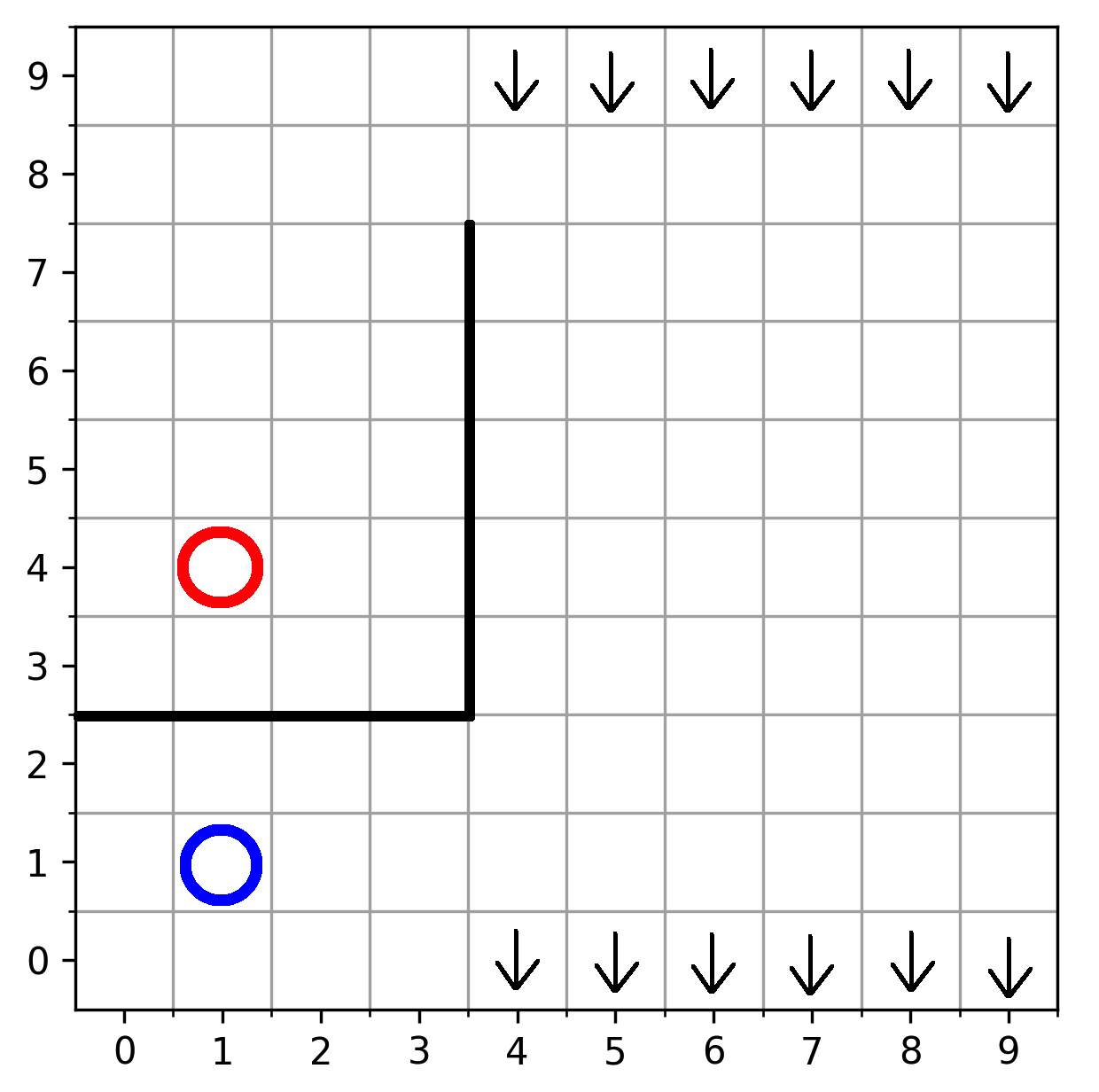}
        \caption{Grid world with wind affecting transitions in last 6 columns}
        \label{fig:windy_grid}
    \end{subfigure}%
    \hfill
    \begin{subfigure}[t]{.36\textwidth}
        \centering
        \includegraphics[width=\linewidth]{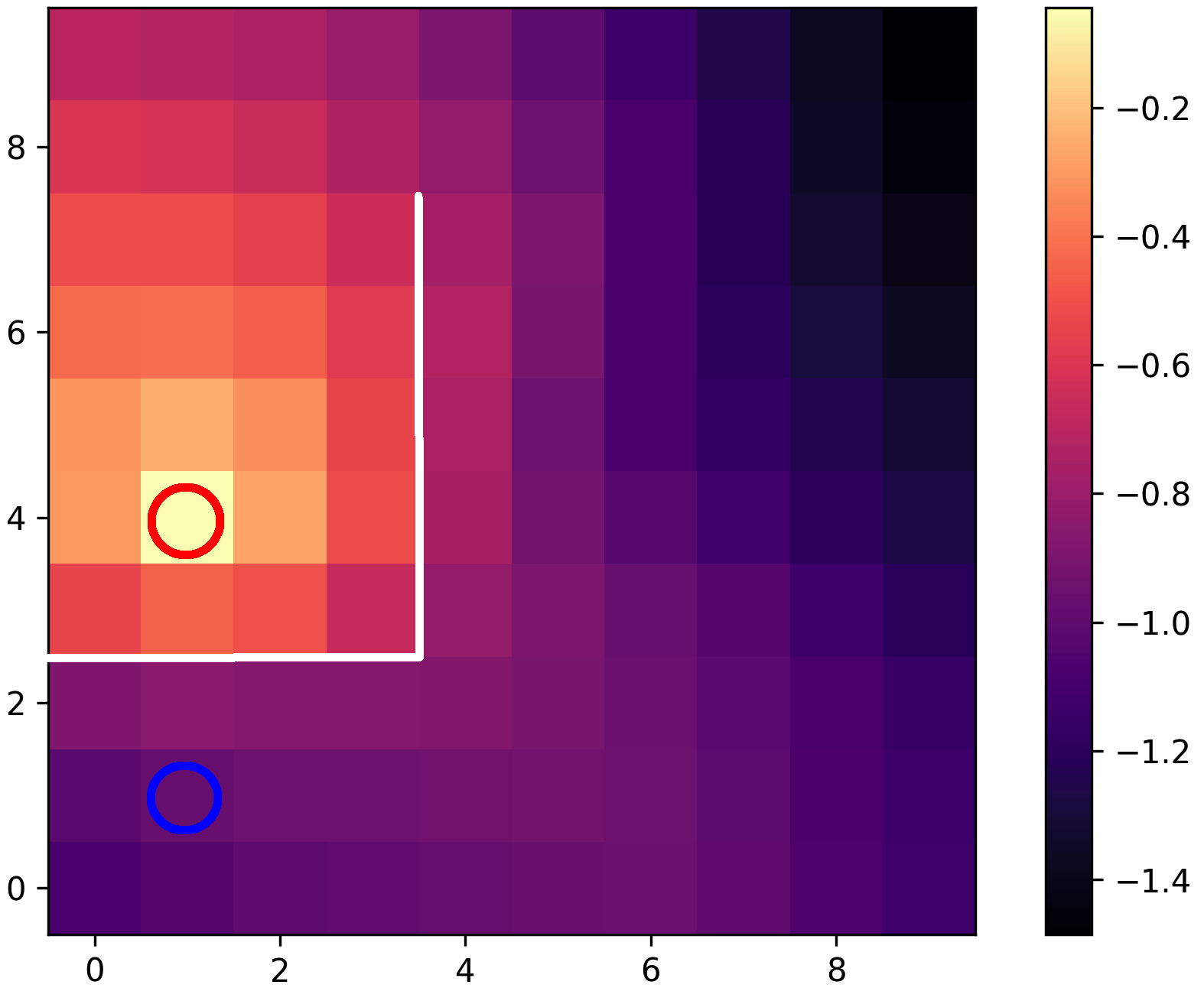}
        \caption{Learned Reward (50 training iterations)}
        \label{fig:windy_reward}
    \end{subfigure}
    \hfill
    \begin{subfigure}[t]{.3\textwidth}
        \centering
        \includegraphics[width=\linewidth]{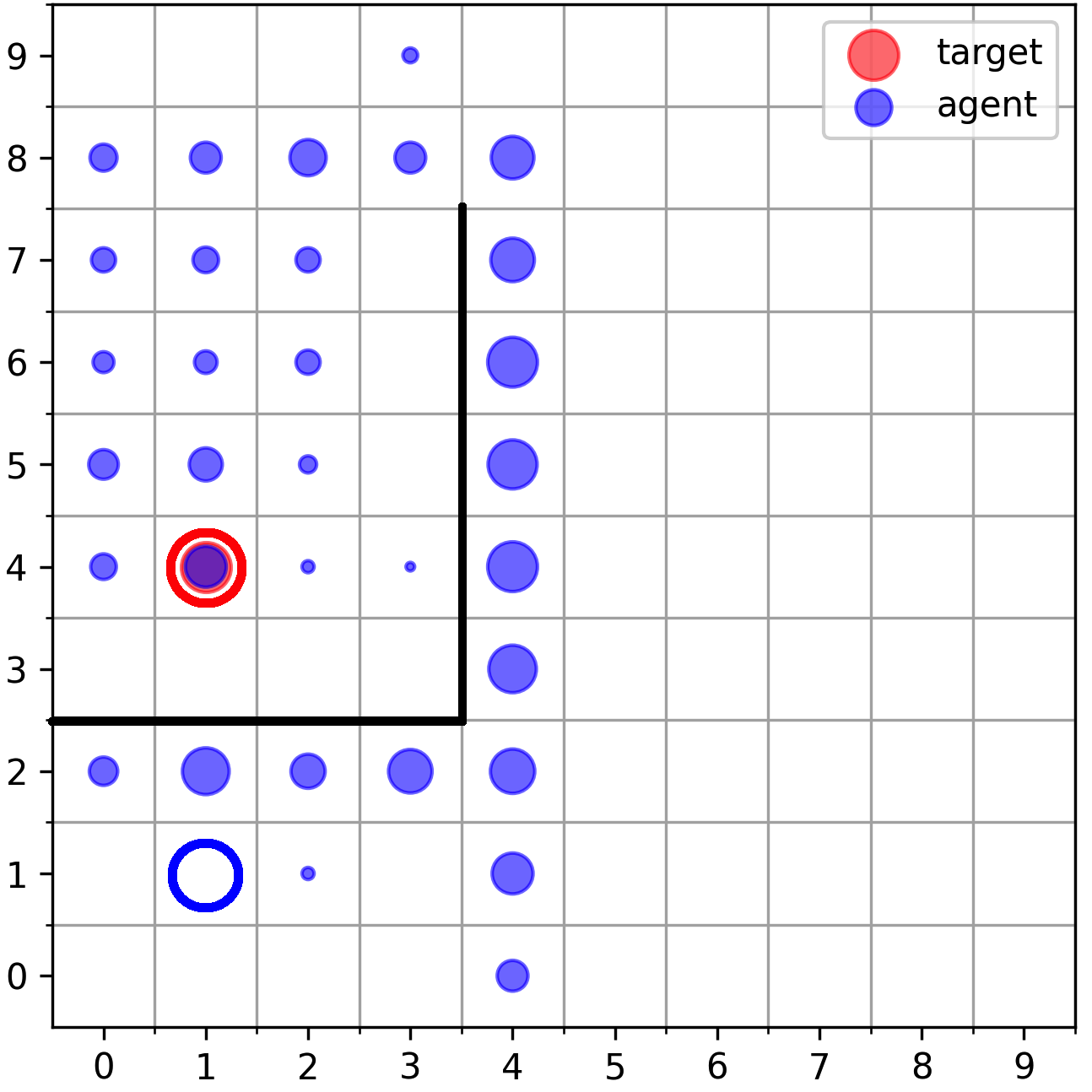}
        \caption{Agent state distribution learning with AIM reward (50 training iterations)}
        \label{fig:windy_policy}
    \end{subfigure}
    \caption{Windy grid world (Figure \ref{fig:windy_grid}) experiments. The columns with arrows at the top and bottom have stochastic and asymmetric transitions induced by wind blowing from the top.
    Learned reward function (Figure \ref{fig:windy_reward}). Reward at each state of the grid world after training for 50 iterations with \textsc{aim}. Hollow red circle indicates the goal state. White lines indicate the walls the agent cannot transition through.
    The agent's state visitation (Figure \ref{fig:windy_policy}): The hollow blue circle indicates agent's start state. The hollow red circle is the goal. Blue bubbles indicate relative time the agent's policy causes it to spend in respective states. Black lines indicate walls.}
    \label{fig:windy_gridworld}
\end{figure}

\paragraph{Basic experiment} The environment is a $10 \times 10$ grid with $4$ discrete actions that take the agent in the $4$ cardinal directions unless blocked by a wall or the edge of the grid.
The agent policy is learned using soft Q-learning \citep{haarnoja2017reinforcement},
with an entropy coefficient of $0.1$ and a discount factor of $\gamma=0.99$. We do not use hindsight goals for this experiment, and use a single buffer with size $5000$ for both the policy as well as the discriminator training. The results are discussed in the main text.
The compute used to conduct these experiments was a personal laptop with an Intel i7 Processor and 16 GB of RAM.

\paragraph{Additional experiments} We conducted variations form the basic experiment in the grid world to show that \aim\ and its novel regularization can learn a reward function which guides the agent to the goal even in the presence of stochastic transitions as well as transitions where the state features vary wildly from one step to the next.

First, we evaluate \aim's ability to learn in the presence of stochastic and asymmetric transitions in a windy version (Figure \ref{fig:windy_grid}) of the above grid world.
Transitions in the last six columns of the grid are affected by a wind blowing from the top.
Actions that try to move upwards only succeed $60\%$ of the time, and actions attempting to move sideways cause a transition diagonally downwards $40\%$ of the time.
Movements downwards are unaffected.
The rest of the experiment is carried out in the same way as above, but with $128$ hidden units in the hidden layer of the agent's Q function approximator (the reward function architecture is unchanged from the previous experiment).
In Figure \ref{fig:windy_gridworld} we see that \aim\; learns a reward function that is still useful and interpretable, and leads to a policy that can confidently reach the goal, regardless of these stochastic and asymmetric transitions.
Notice the effect of the stochastic transitions in the increased visitation in the sub-optimal states in the bottom two rows of column number $4$.

The next experiment tests what happens when the transition function causes the agent to jump between states where the state features vary sharply.
As an example consider a toroidal grid world, where if an agent steps off one side of the grid it is transported to the other side.
The distance function here should be smooth across such transitions, but might be hampered by the sharp change in input features.
In Figure \ref{fig:toroid} we see show the policy and reward for a $10\times10$ toroidal grid world with start state at $(2,2)$ and goal at $(7, 7)$.
Transitions are deterministic but wrap around the edges of the grid as described above: a \textbf{down} action in row $0$ will transport the agent to the same column but row $9$.
The start and the goal state are set up so that there are multiple optimal paths to the goal.
The entropy maximizing soft Q-learning algorithm should take these paths with almost equal probability.
From Figure \ref{fig:toroid} it is evident that \aim\; learns a reward function that is smooth across the actual transitions in the environment and allows the agent to learn a Q-function that places near equal mass on multiple trajectories.

\begin{figure}[b]
    \centering
    \begin{subfigure}{.34\textwidth}
    \includegraphics[width=\textwidth]{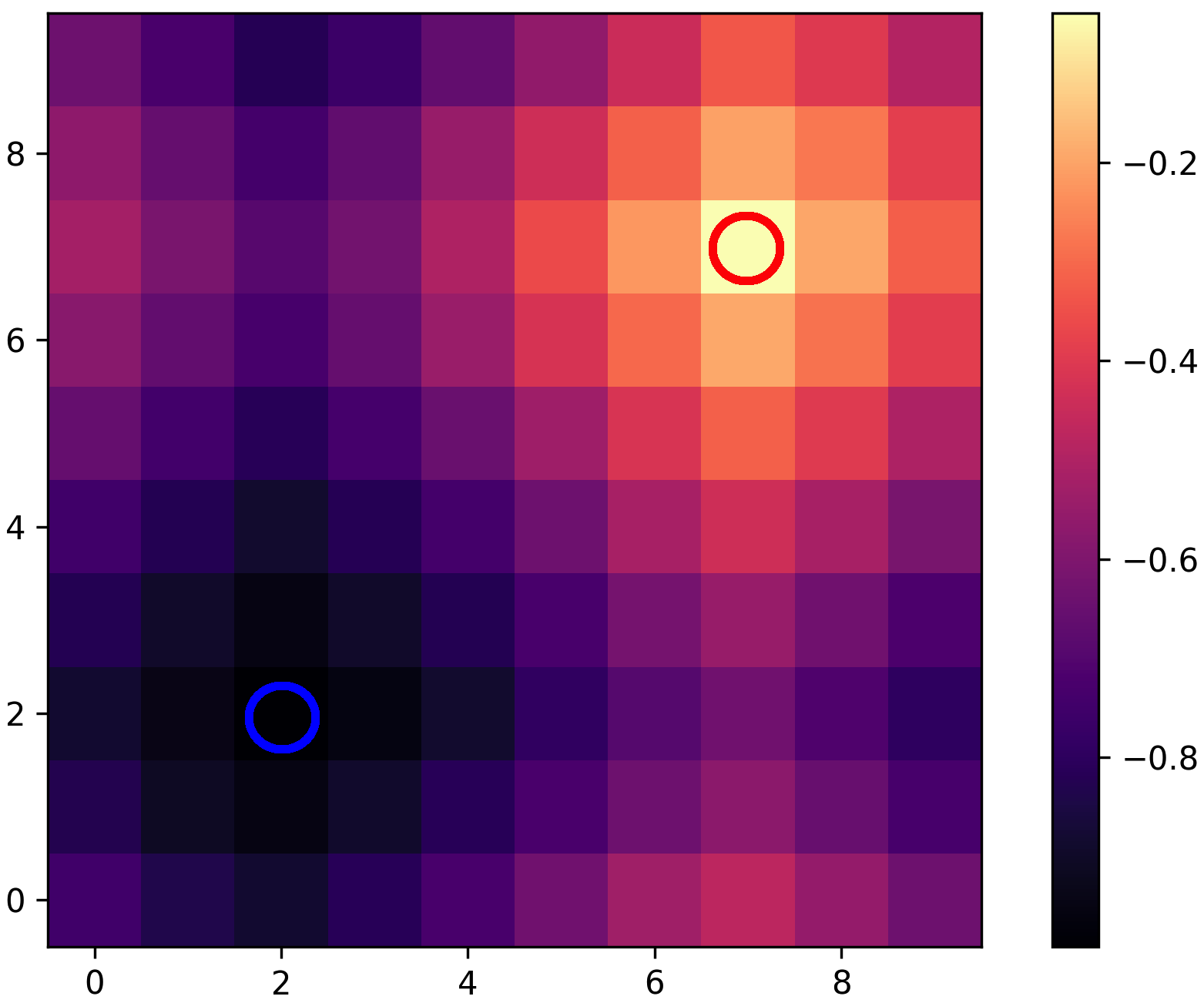}
    \caption{Reward function with \aim}
    \label{fig:r_toroid}
    \end{subfigure}%
    \hfill
    \begin{subfigure}{.29\textwidth}
    \includegraphics[width=\textwidth]{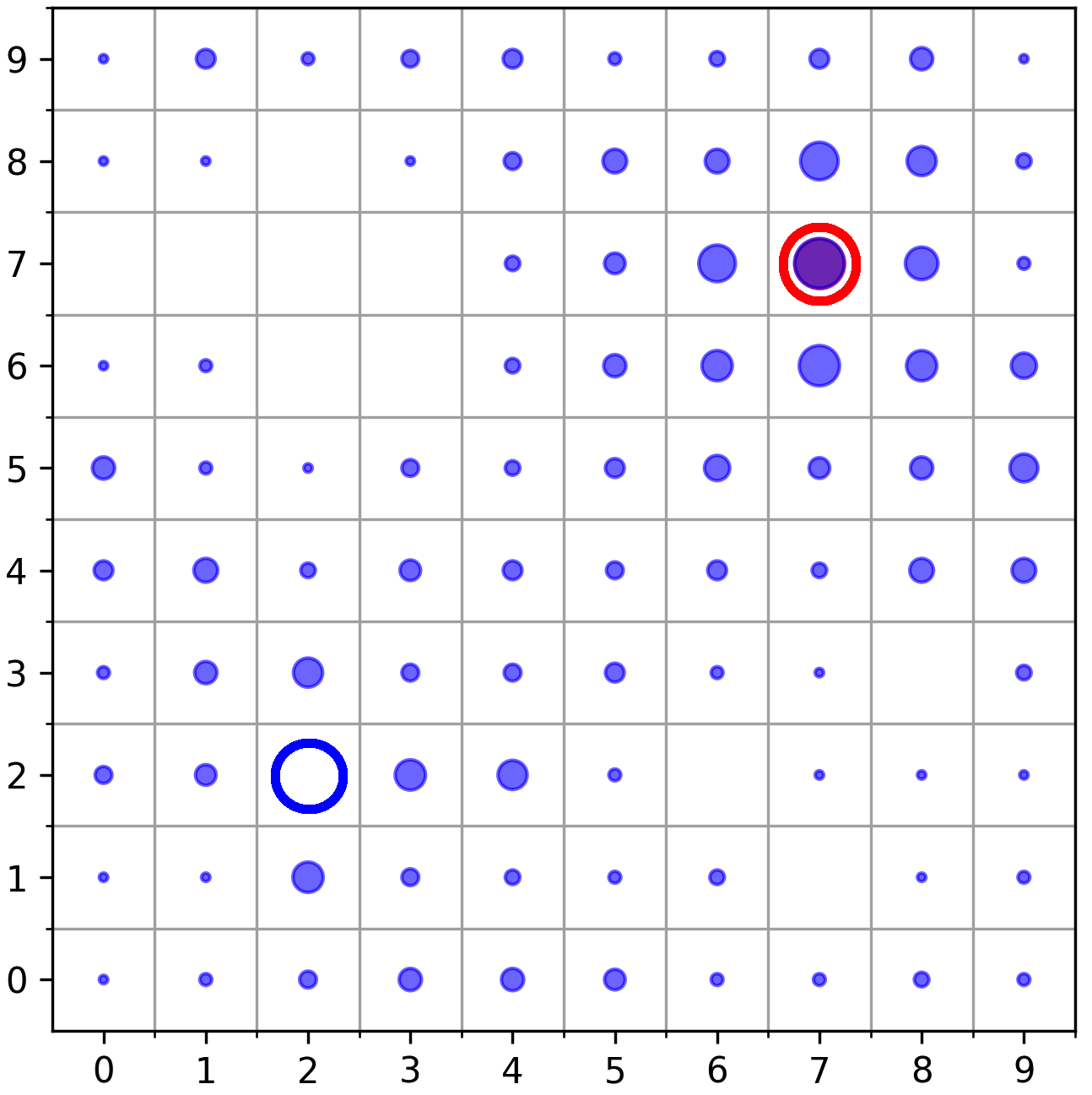}
    \caption{Policy distribution under \aim}
    \label{fig:p_toroid}
    \end{subfigure}%
    \hfill
    \begin{subfigure}{.34\textwidth}
    \includegraphics[width=\textwidth]{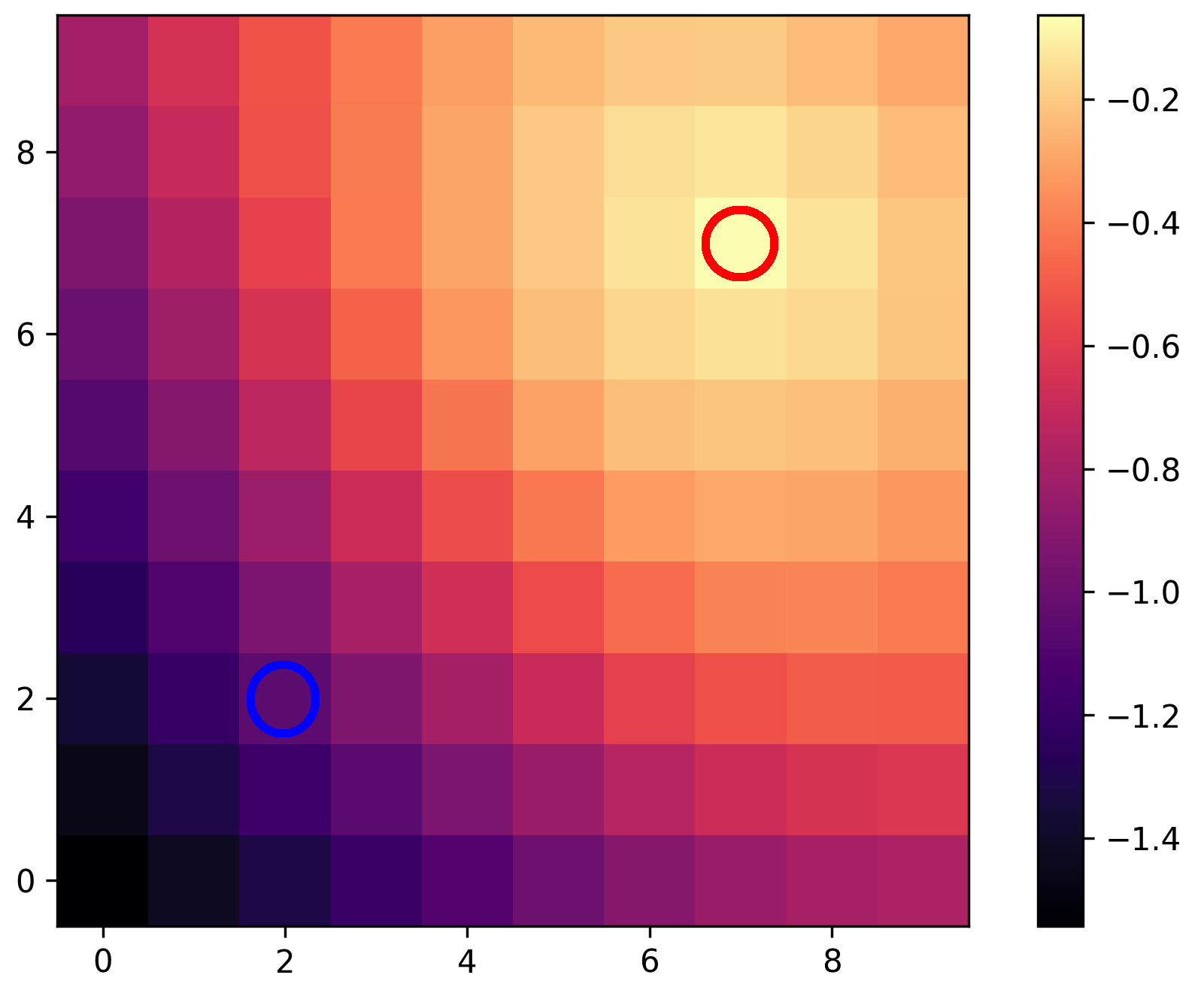}
    \caption{Reward function estimated with WGAN loss}
    \label{fig:r_gnorm}
    \end{subfigure} 
    \caption{The reward function (Figure \ref{fig:r_toroid}) learned with \aim\; and subsequent policy distribution (Figure \ref{fig:p_toroid}) in a toroidal grid world, where an agent can transition from one edge of the grid across to the other. The hollow blue circle denotes the start state and the hollow red circle is the goal state. The reward function respects the sharp transitions from one end of the grid to the other. Conversely, if the reward function is learned using the WGAN objective \citep{gulrajani_improved_2017} (Figure \ref{fig:r_gnorm}), it does not respect the environment dynamics.}
    \label{fig:toroid}
\end{figure}

Finally, we compare learning with \aim\ to the baselines mentioned in Section \ref{sec:exp}.
RND, SMiRL, and MC were implemented and debugged on the grid world domain with a goal that is easier to reach before being used on the Fetch robot tasks.
Hyper-parameters for the algorithms in both domains were determined through sweeps.
In the Fetch domains, the hyperparameters for all three new baselines were decided on through sweeps on the FetchReach task, similar to how they were evaluated for \aim\ and the other baselines.

Figure \ref{fig:add_grid} shows the results of executing these additional baselines on the grid world domain we use to motivate \aim.
All the plots are taken after the techniques have had the same number of training iterations.
However none of the baselines reach the goal even after providing additional time.
We show the negative L2 distance to goal as a reward in the grid world domain to highlight that the DiscoRL \citep{Nasiriany:EECS-2020-151} objective should not be considered equivalent to an oracle of the distance to goal.
Note that RND (Figure \ref{fig:grid_rnd}) explores most of the larger room early on, and then converges to the state distribution seen in the figure when it does not encounter the task reward.
The SMiRL reward encourages the agent to minimize surprise, and the policy trained with this reward keeps the agent in the bottom left near its start state (Figure \ref{fig:grid_smirl}).

\begin{figure}[ht]
    \centering
    \begin{subfigure}{.33\textwidth}
    \includegraphics[width=\textwidth]{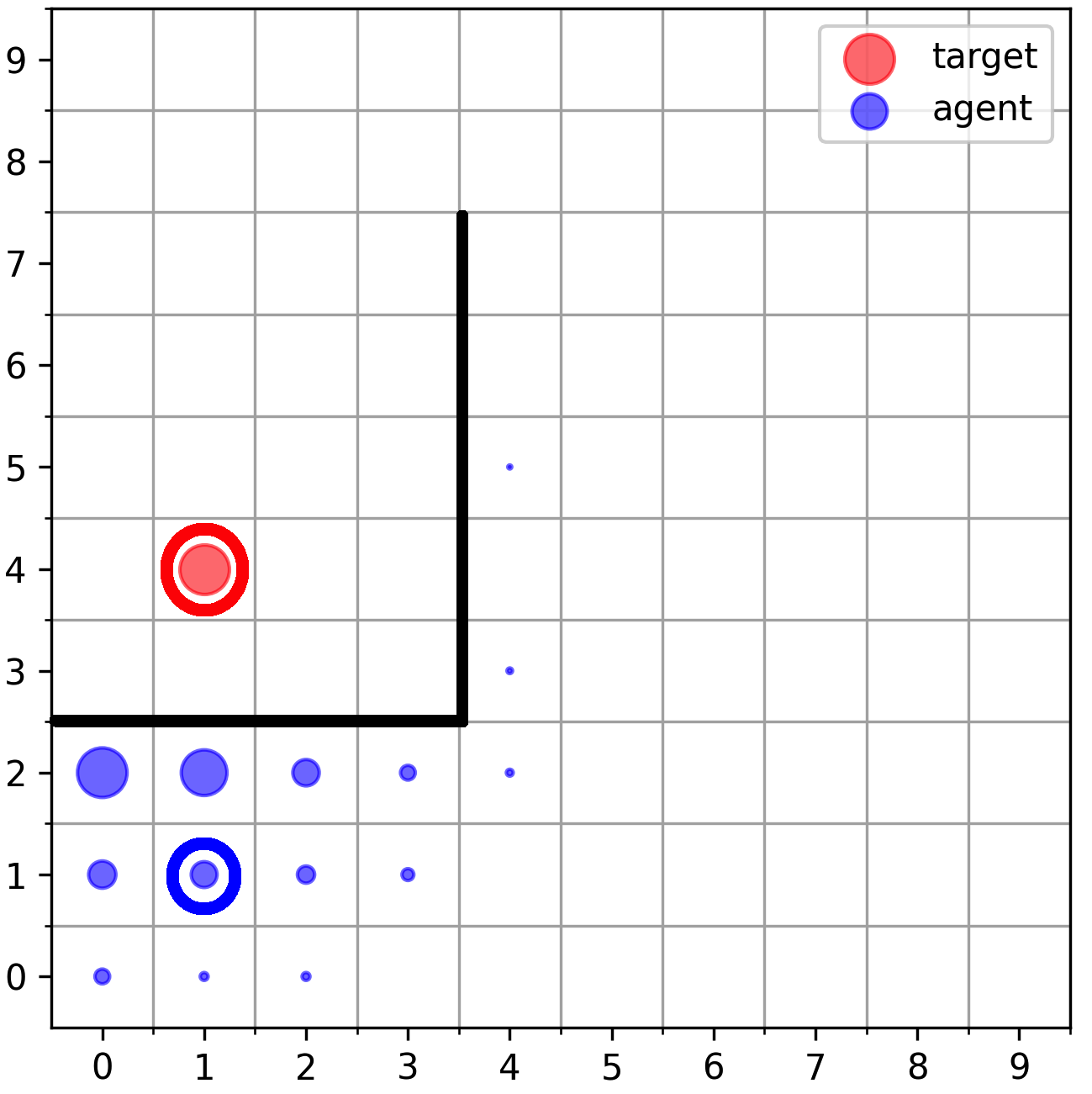}
    \caption{State visitation with $- L2$ reward (DiscoRL)}
    \label{fig:grid_disco}
    \end{subfigure}%
    \hfill
    \begin{subfigure}{.33\textwidth}
    \includegraphics[width=\textwidth]{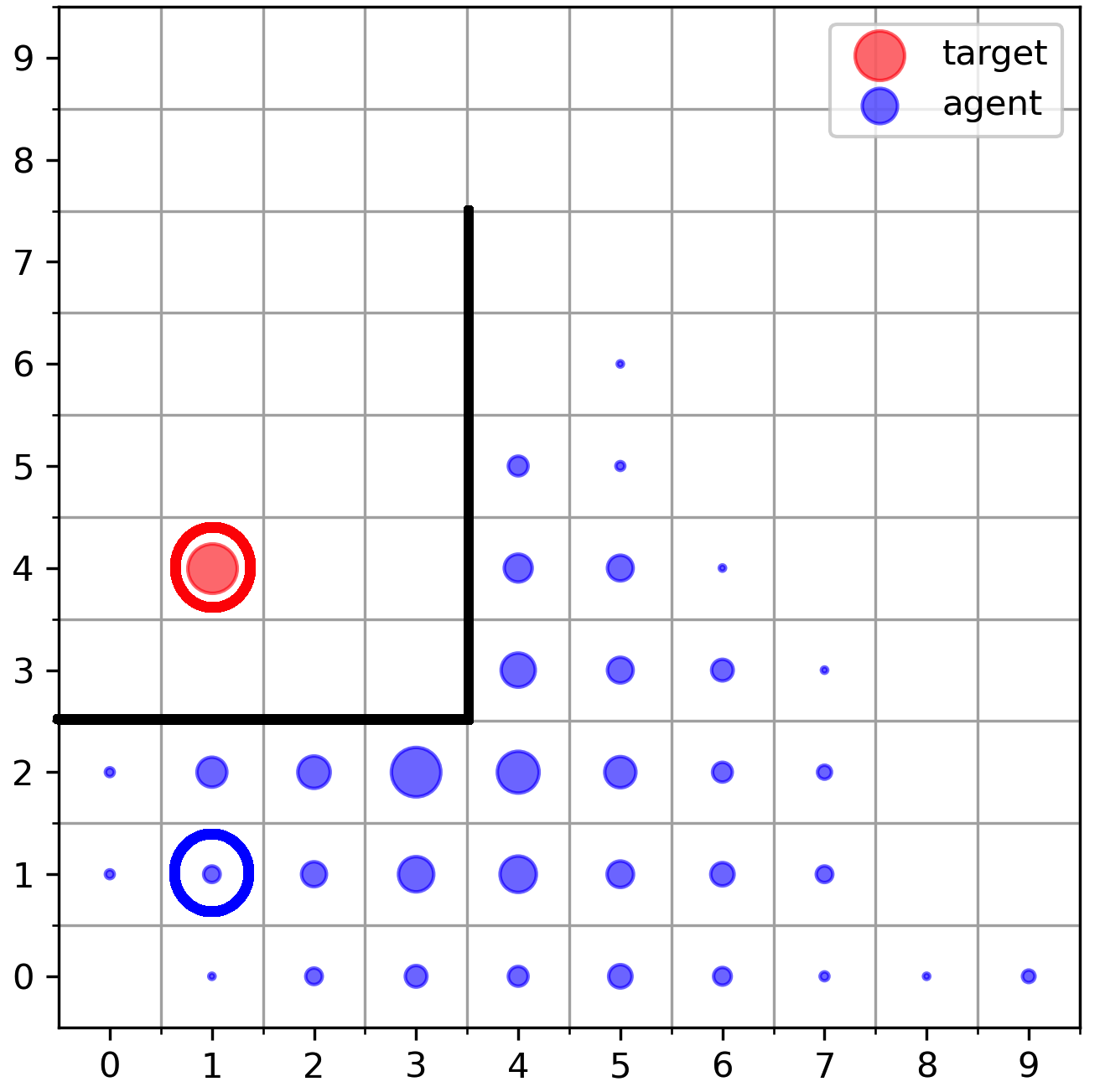}
    \caption{State visitation with MC distance estimation}
    \label{fig:grid_ddl}
    \end{subfigure}%
    \hfill
    \begin{subfigure}{.33\textwidth}
    \includegraphics[width=\textwidth]{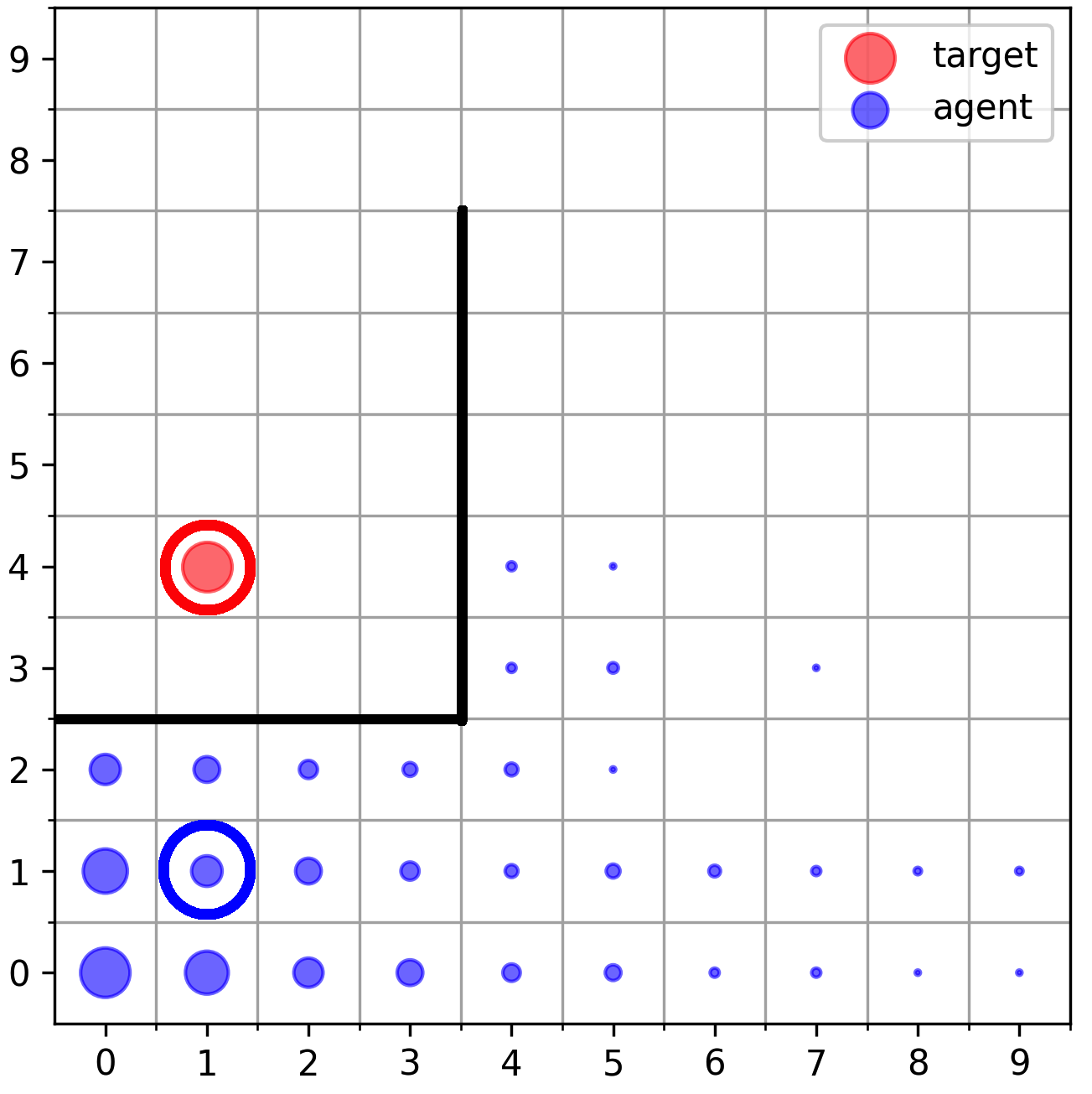}
    \caption{State visitation with added RND reward}
    \label{fig:grid_rnd}
    \end{subfigure}
    \\
    \centering
    \begin{subfigure}{.33\textwidth}
    \includegraphics[width=\textwidth]{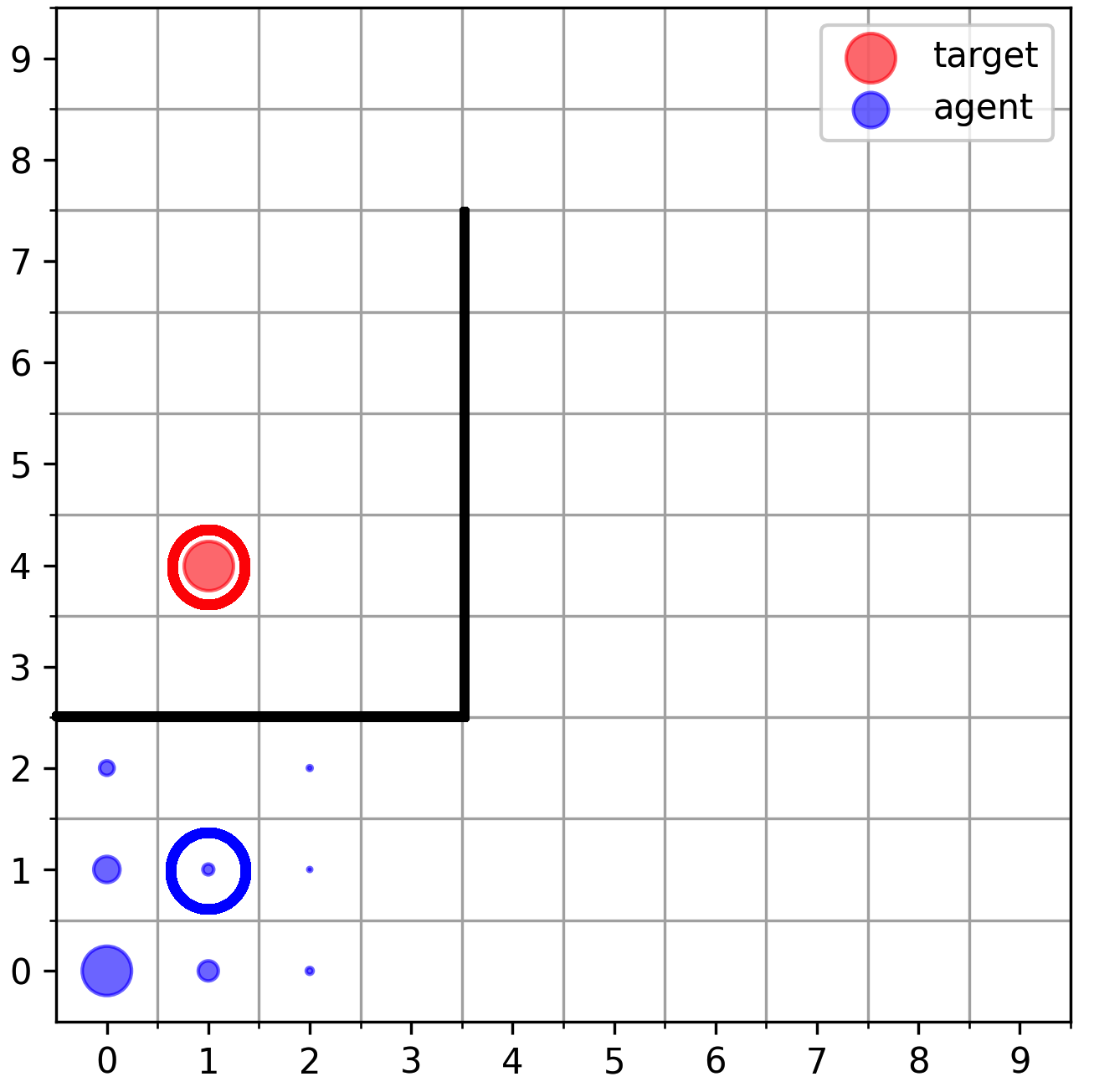}
    \caption{State visitation with SMiRL rewards}
    \label{fig:grid_smirl}
    \end{subfigure}%
    \hfill
    \begin{subfigure}{.33\textwidth}
    \includegraphics[width=\textwidth]{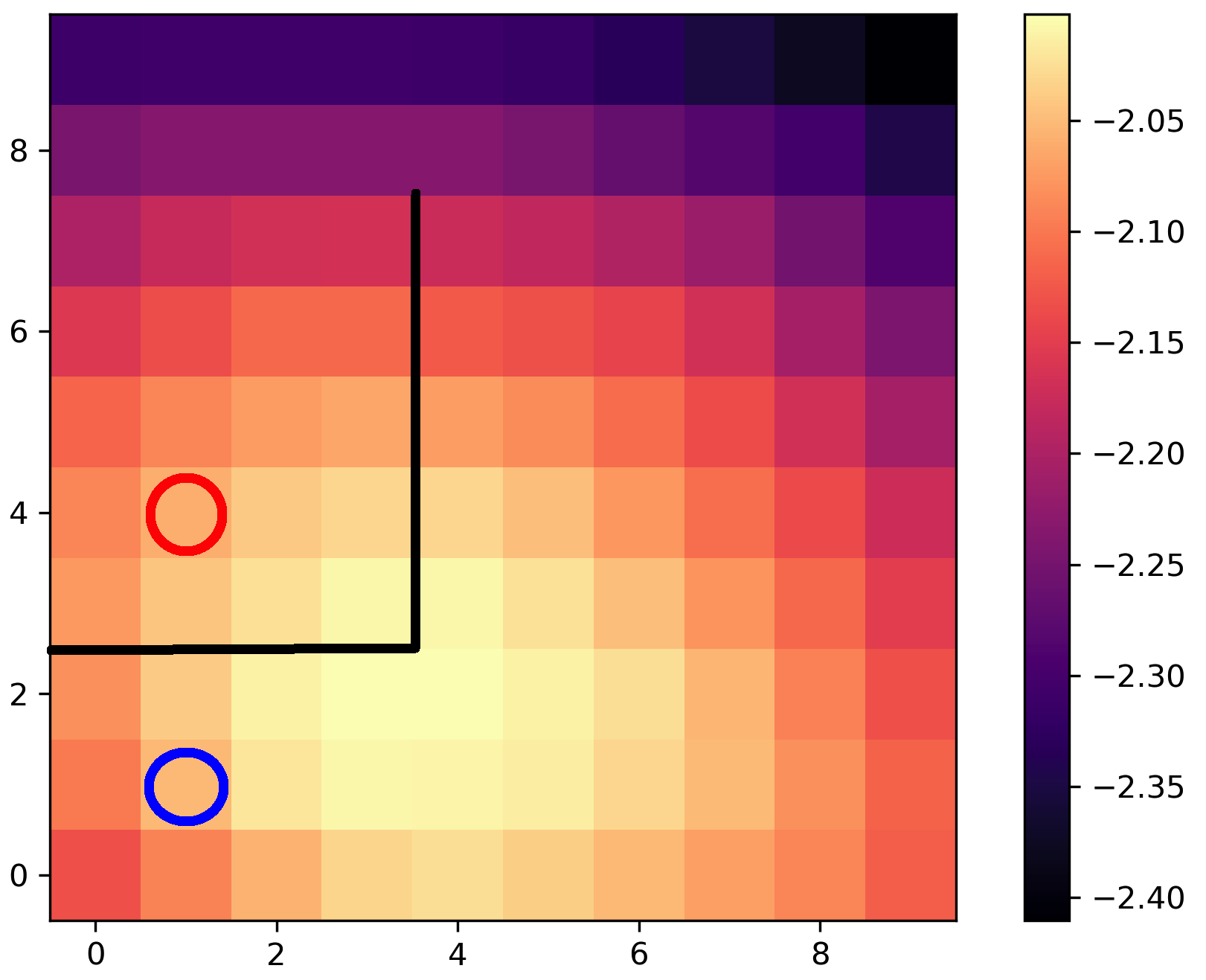}
    \caption{DDL reward function}
    \label{fig:grid_rew_ddl}
    \end{subfigure}%
    \hfill
    \begin{subfigure}{.33\textwidth}
    \includegraphics[width=\textwidth]{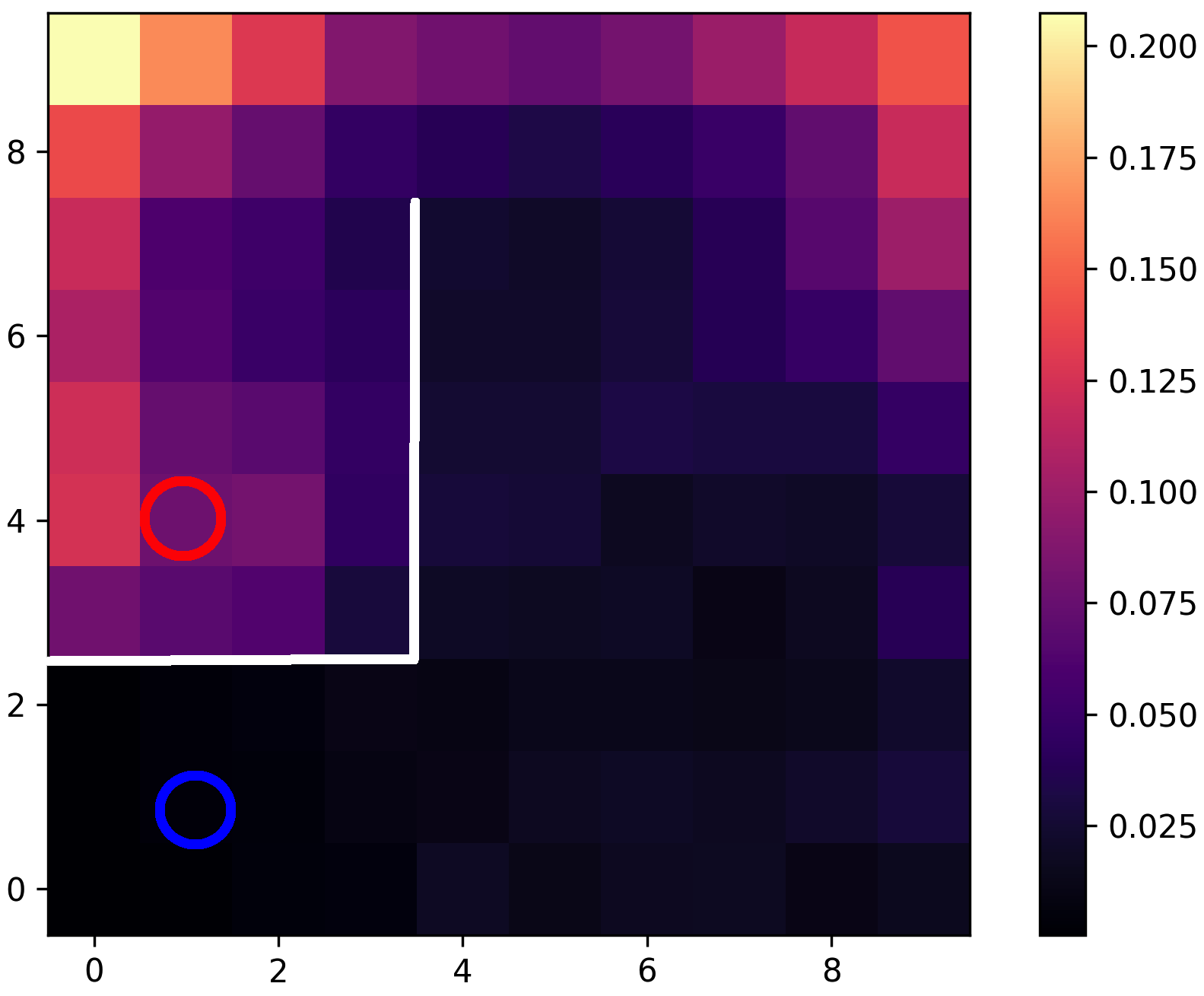}
    \caption{RND reward function}
    \label{fig:grid_rew_rnd}
    \end{subfigure}
    \caption{The state of the state visitation and reward functions for the new baselines. For camparison, Figure \ref{fig:grid_policy} shows the state visitation of policy trained using \aim. All algorithms are compared after 100 training iterations.}
    \label{fig:add_grid}
\end{figure}

\begin{figure}[t]

    \centering
    \begin{subfigure}{.48\textwidth}
    \includegraphics[width=\textwidth]{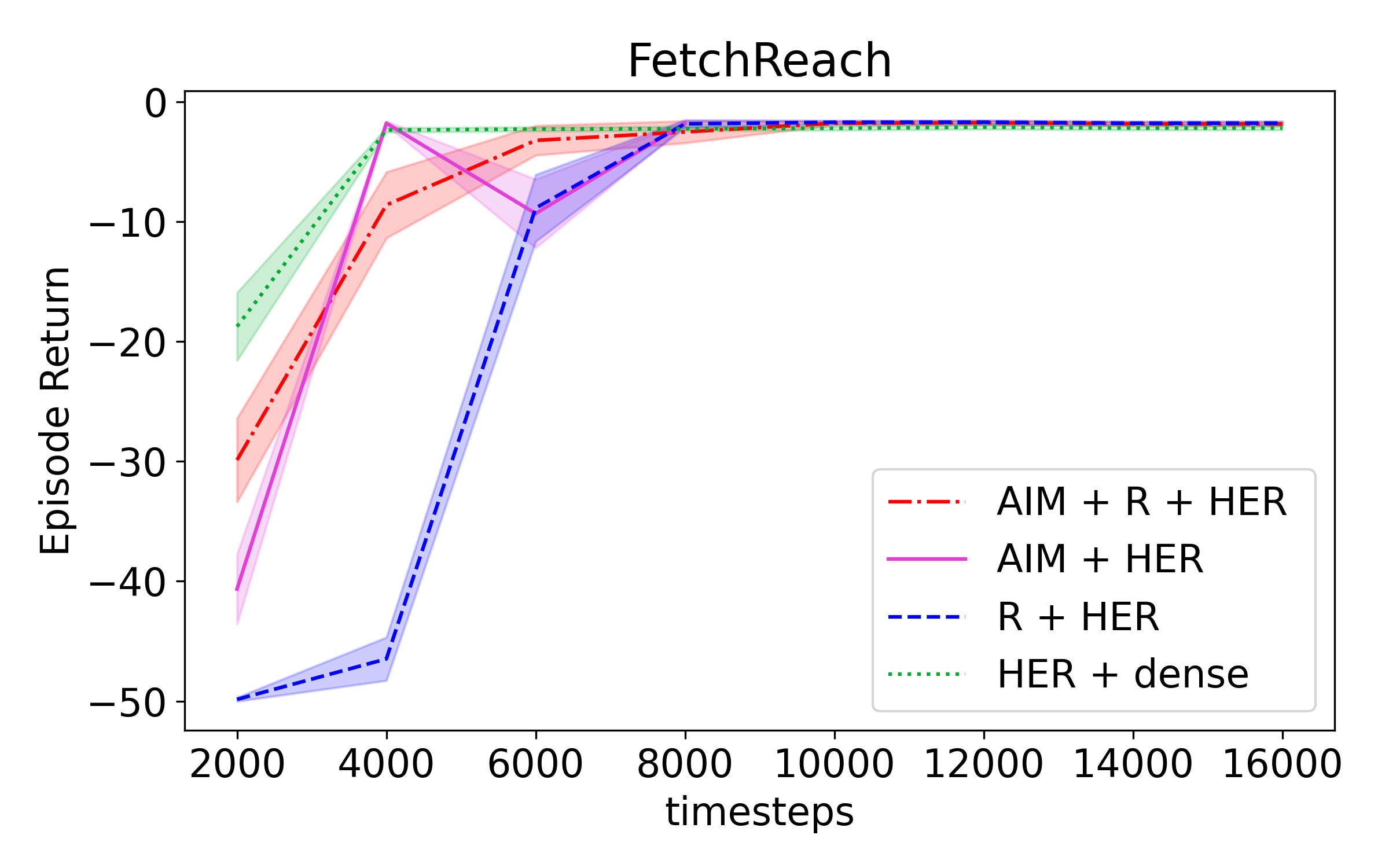} 
    \caption{Reach}
    \label{fig:app_reach}
    \end{subfigure}%
    \hfill
    \begin{subfigure}{.48\textwidth}
    \includegraphics[width=\textwidth]{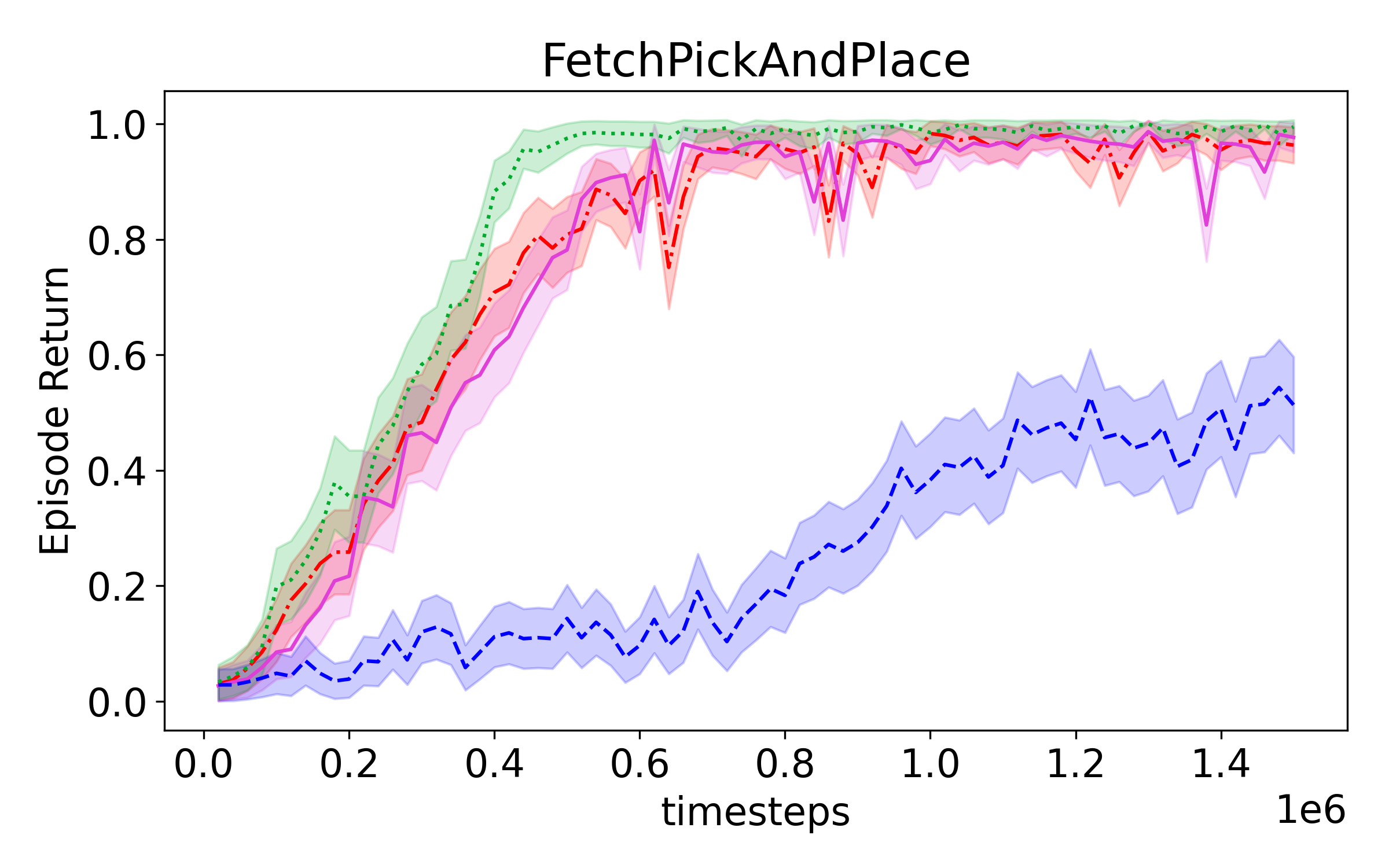}
    \caption{Pick and Place}
    \label{fig:app_pick}
    \end{subfigure}%
    \\
    \begin{subfigure}{.48\textwidth}
    \includegraphics[width=\textwidth]{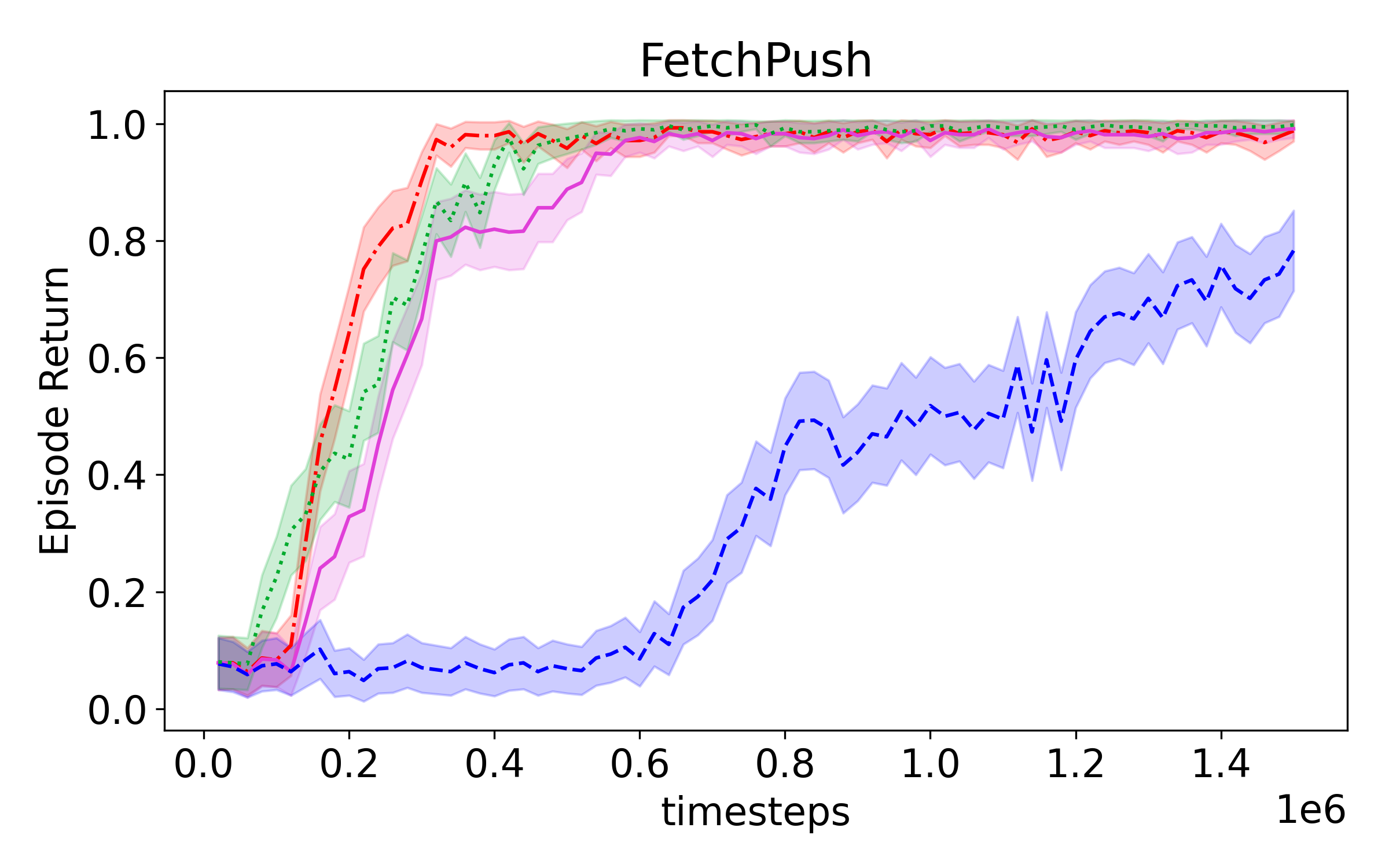}
    \caption{Push}
    \label{fig:app_push}
    \end{subfigure}%
    \hfill
    \begin{subfigure}{.48\textwidth}
    \includegraphics[width=\textwidth]{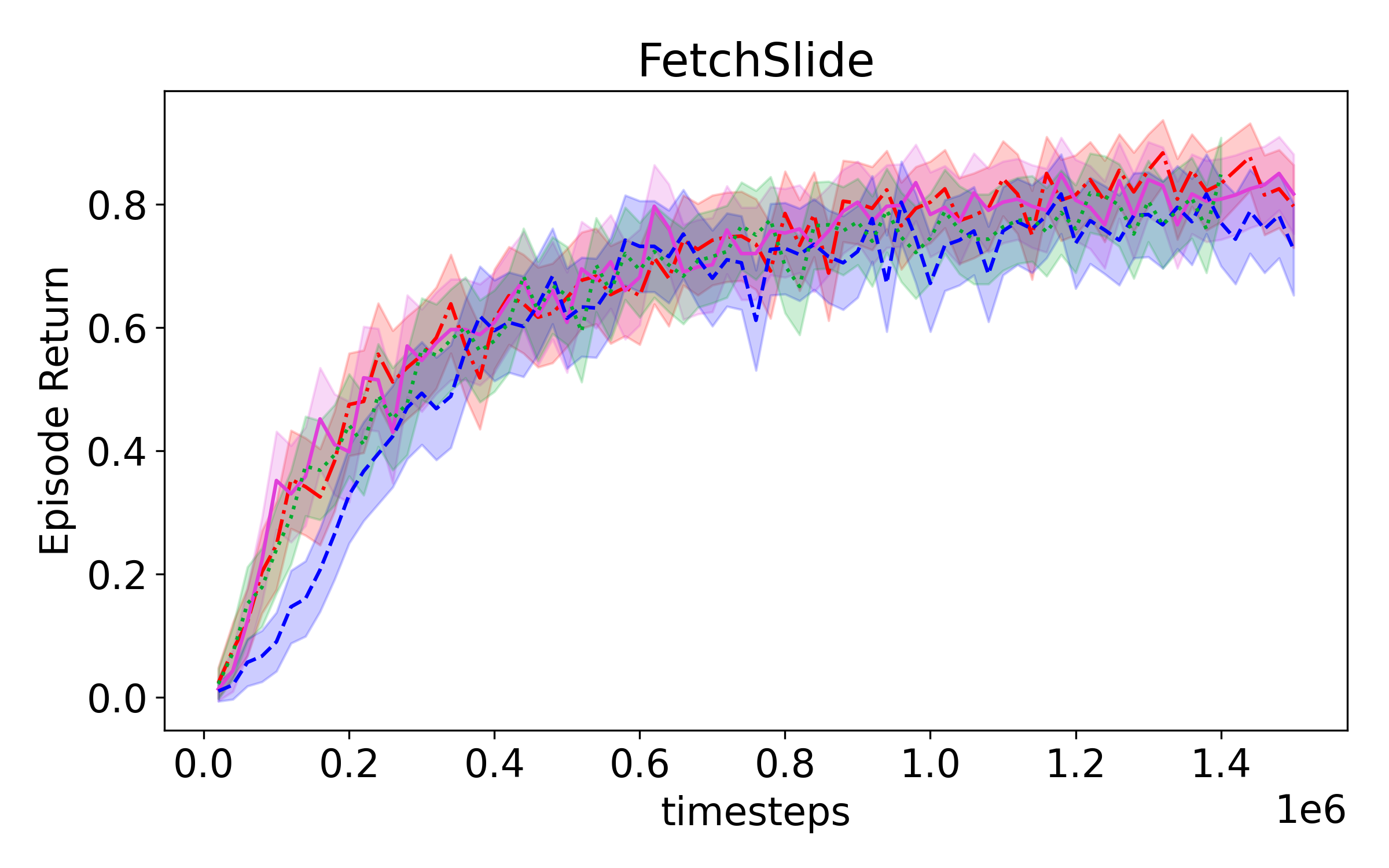}
    \caption{Slide}
    \label{fig:app_slide}
    \end{subfigure}
    \caption{Comparing [\aim\ + \textsc{her}] with an additional baseline which also uses the external task reward [\aim\ + R + \textsc{her}]. The additional grounding provided by the external task reward allows the agent's learning to accelerate even further.}
    \label{fig:app_fetch}
\end{figure}

\section{Statistical Analysis of the Results on Fetch Robot Tasks} \label{app:fetch_test}

To compare the performance of each method with statistical rigor, we used a repeated measures ANOVA design for binary observation where an observation is successful if an agent reaches the goal within an episode. We then conducted a Tukey test to compare the effects of each method, i.e., the estimated odds of reaching the goal given the algorithm. The goal of the statistical analysis presented here is twofold
\begin{enumerate}[itemsep=0pt]
    \item Separate the uncertainty on the performance of each method from the variation due to random seeds.
    \item Adjust the probability of making a false discovery due to multiple comparisons. This extra step is necessary to avoid detecting a large fraction of falsely ``significant" differences since typical tests are designed to control the error rate of only one experiment.
\end{enumerate}

The data for statistical analysis comes from $N_{\text{episodes}}=100$ evaluation episodes per each one of $N_{\text{seeds}}=6$ seeds. For all environments but FetchReach, these data is collected after 1 million environment interactions; and for FetchReach it is taken after 2000 interactions. 

The repeated measures ANOVA design is formulated as a mixed effects generalized linear model and fitted separately for each one of the four environments
\begin{align*}
y_{ijk} &\overset{\text{\tiny iid}}{\sim} \mathrm{Bernoulli}(p_{ij}), \quad &k\in\{1,\hdots, N_{\text{episodes}}\} \\
\text{logit}(p_{ij}) &= r_{\text{seed}_i} + \beta_{\text{algorithm}_j}\quad & i\in\{1,\hdots,N_{\text{seeds}}\}, 
j\in\{1,\hdots,N_{\text{algorithms}}\} \\
r_{\text{seed}_i} &\overset{\text{\tiny iid}}{\sim} \sim \mathrm{Normal}(0, \sigma^2)  \quad &
\end{align*}
The variation due to the seed effects is measured by $\sigma^2$, whereas the uncertainty about the odds of reaching the goal using each algorithm is measured by the standard errors of the coefficients $\beta_{\text{algorithm}_j}$. The Tukey test evaluates all null hypotheses $H_0\colon \beta_{\text{algorithm}_j} = \beta_{\text{algorithm}_{j'}}$ for all combinations of $j,j'$. To adjust for multiple comparisons each Tukey tests uses the Holm method. Since we are also doing a Tukey test for each environment, we further apply a Bonferroni adjustment with a factor of four. These types of adjustments are fairly common for dealing with multiple comparison in the literature of experimental design; the interested reader may consult \citep{montgomery2017design}.

The results, shown in Table \ref{tab:anova}, signal strong statistical evidence of the improvements from using the \aim\ learned rewards. In three of the four environments \aim\ and \aim + \textsc{r} have similar odds of reaching the goal as the dense shaped reward ($H_0$ is not rejected,) and in all four environments \aim\ and \aim + \textsc{r} have statistically significant higher odds of reaching the goal than the sparse reward ($H_0$ is rejected and $\beta$ is higher.) 
  
 \begin{table}[thb]
     \centering
     \footnotesize
     \begin{tabular}{c|r|r|r|r}
    Contrast & Slide & Push & PickAndPlace &  Reach\\ \midrule
    $\beta_{\text{\aim+\textsc{r}}} - \beta_{\text{\textsc{her}+dense}}$
    & 0.34 (0.14) & -1.74 (0.77) & -0.10 (0.45) & *-3.43 (0.34)
    \\
   $\beta_{\text{\aim}} - \beta_{\text{\textsc{her}+dense}}$
   &  0.21 (0.14) & -2.19 (0.75) & *-1.50 (0.37) & *-5.01 (0.35)
   \\ \midrule
    $\beta_{\text{\aim}+\textsc{r}} - \beta_{\text{\textsc{her}+sparse}}$ 
    &  *0.69 (0.13) & *5.32 (0.35) & *4.71 (0.33) & *4.75 (0.25)
    \\
       $\beta_{\text{\aim}} - \beta_{\text{\textsc{her}+sparse}}$
       &  *0.57 (0.13) & *4.86 (0.30) & *3.31 (0.19) & *3.17 (0.24)
     \end{tabular}
     
     \caption{Results of the Tukey test on the evaluation of Fetch tasks. The table entries are log odds ratios with standard deviations shown in parentheses. Positive values mean that \aim\ or \aim+\textsc{r} perform better than the method with negative sign in the contrast and viceversa. Asterisks mark statistical significance at 95\%. If there is no asterisk, then $H_0$ is not rejected in which case the differences could be due to random chance.}
     \label{tab:anova}
 \end{table}

\section{Details of Experiments on Fetch Robot} \label{app:fetch}

The Fetch robot domain in OpenAI gym has four tasks available for testing.
They are named Reach, Push, Slide, and Pick And Place.
The Reach task is the simplest, with the goal being the 3-d coordinates where the end effector of the robot arm must be moved to.
The Push task requires pushing an object from its current position on the table to the given target position somewhere else on the table.
Slide is similar to Push, except the coefficient of friction on the table is reduced (causing pushed objects to slide) and the potential targets are over a larger area, meaning that the robot needs to learn to hit objects towards the goal with the right amount of force.
Finally, Pick And Place is the task where the robot actuates it's gripper, picks up an object from its current position on the table and moves it through space to a given target position that could be at some height above the table.
The goal space for the final three tasks are the required position of the object, and the goal the current state represents is the current position of that object.

Next, we note the hyperparameters used for various baselines as well as our implementation.
The names of the hyperparameters are as specified in the stable baselines repository and used in the RL Zoo \citep{rl-zoo} codebase which we use for running experiments.
Both the stable baselines repository and RL Zoo are available under the MIT license.
These experiments were run on a compute cluster with each experiment assigned an Nvidia Titan V GPU, a single CPU and 12 GB of RAM.
Each run of the TD3 baseline \textsc{her} + \textsc{r} or \textsc{her} + dense required $18$ hours to execute, and each run which included \aim\ required $24$ hours to complete execution.

TD3 \citep{fujimoto2018TD3}, like its predecessor DDPG \citep{lillicrap2015continuous}, suffers from the policy saturating to extremes of its parameterization.
\citet{hausknecht2016deep} have suggested various techniques to mitigate such saturation.
We use a quadratic penalization for actions that exceed $80\%$ of the extreme value at either end, which is sufficient to not hurt learning and prevent saturation.
Assuming the policy network predicts values between $-1$ and $1$ (as is the case when using the tanh activation function), the regularization loss is:
\begin{align*}
    L_a = \frac{1}{N} \sum_{i=1}^{N} \left[max(|\pi_\theta(s_i)| - 0.8, 0)\right] ^ 2
\end{align*}
where $N$ is the mini-batch size and $s_i$ is the state for the $i^{\text{th}}$ transition in the batch.

The other modification made to the stable baselines code is to use the Huber loss instead of the squared loss for Q-learning.

For evaluation, in the Reach domain the agent policy is evaluated for 100 episodes every 2000 steps.
For the other three domains, the experiment is run for 1 million timesteps, and evaluated at every 20,000 steps for 100 episodes.

\subsection{TD3 and HER (R + HER)}

 \begin{tabular}{||c | c||} 
 \hline
 Hyperparameter & Value \\ [0.5ex] 
 \hline\hline
 n_sampled_goal & 4 \\ 
 \hline
 goal_selection_strategy & future  \\
 \hline
 buffer_size & $10^6$ \\
 \hline
 batch_size & 256 \\
 \hline
 $\gamma$ (discount factor) & $0.95$ \\
 \hline
 random_exploration & 0.3 \\
 \hline
 target_policy_noise & 0.2 \\
 \hline
 learning_rate & $1^{-3}$ \\
 \hline
 noise_type & normal \\
 \hline
 noise_std & $0.2$ \\
 \hline
 MLP size of agent policy and Q function & $[256, 256, 256]$ \\
 \hline
 learning_starts & 1000 \\
 \hline
 train_freq & 10 \\
 \hline
 gradient_steps & 10 \\
 \hline
 $\tau$ (target policy update rate) & 0.05 \\ [1ex] 
 \hline
\end{tabular}

\subsection{Dense reward TD3 and HER (dense + HER)}

 \begin{tabular}{||c | c||} 
 \hline
 Hyperparameter & Value \\ [0.5ex] 
 \hline\hline
 n_sampled_goal & 4 \\ 
 \hline
 goal_selection_strategy & future  \\
 \hline
 buffer_size & $10^6$ \\
 \hline
 batch_size & 256 \\
 \hline
 $\gamma$ (discount factor) & $0.95$ \\
 \hline
 random_exploration & 0.3 \\
 \hline
 target_policy_noise & 0.2 \\
 \hline
 learning_rate & $1^{-3}$ \\
 \hline
 noise_type & normal \\
 \hline
 noise_std & $0.2$ \\
 \hline
 MLP size of agent policy and Q function & $[256, 256, 256]$ \\
 \hline
 learning_starts & 1000 \\
 \hline
 train_freq & 100 \\
 \hline
 gradient_steps & 200 \\
 \hline
 policy_delay & 5 \\
 \hline
 $\tau$ (target policy update rate) & 0.05 \\ [1ex] 
 \hline
\end{tabular}

\subsection{TD3 and HER with AIM (AIM + HER) and (AIM + R + HER)}

 \begin{tabular}{||c | c||} 
 \hline
 Hyperparameter & Value \\ [0.5ex] 
 \hline\hline
 n_sampled_goal & 4 \\ 
 \hline
 goal_selection_strategy & future  \\
 \hline
 buffer_size & $10^6$ \\
 \hline
 batch_size & 256 \\
 \hline
 $\gamma$ (discount factor) & $0.9$ \\
 \hline
 random_exploration & 0.3 \\
 \hline
 target_policy_noise & 0.2 \\
 \hline
 learning_rate & $1^{-3}$ \\
 \hline
 noise_type & normal \\
 \hline
 noise_std & $0.2$ \\
 \hline
 MLP size of agent policy and Q function & $[256, 256, 256]$ \\
 \hline
 learning_starts & 1000 \\
 \hline
 train_freq & 100 \\
 \hline
 gradient_steps & 200 \\
 \hline
 disc_train_freq & 100 \\
 \hline
 disc_steps & 20 \\
 \hline
 $\tau$ (target policy update rate) & 0.1 \\ [1ex] 
 \hline
\end{tabular}

\end{document}